\documentclass{article} %
\usepackage{iclr2026_conference,times}

\usepackage{amsmath,amsfonts,bm}

\def\eqref#1{equation~\ref{#1}}

\def\1{\bm{1}}

\DeclareMathAlphabet{\mathsfit}{\encodingdefault}{\sfdefault}{m}{sl}
\SetMathAlphabet{\mathsfit}{bold}{\encodingdefault}{\sfdefault}{bx}{n}

\newcommand{\R}{\mathbb{R}}

\DeclareMathOperator*{\argmin}{arg\,min}

\usepackage[colorlinks=true]{hyperref}
\usepackage{xcolor} %
\definecolor{custom_blue}{rgb}{0.1, 0.3, 0.4}
\hypersetup{
    colorlinks,
    citecolor=custom_blue,
    linkcolor=custom_blue,
    urlcolor=custom_blue,
}
\usepackage{url}
\usepackage[capitalize,noabbrev, nameinlink]{cleveref}

\usepackage{microtype}      %
\usepackage{amssymb, amsthm}%
\usepackage{pifont}%
\usepackage{algorithm}%
\usepackage{algorithmicx}
\usepackage[noend]{algpseudocode}
\usepackage{booktabs}
\usepackage{multirow}
\usepackage{amssymb, dsfont}

\newtheorem{lemma}{Lemma}

\newcommand{\norm}[1]{{\Vert #1 \Vert }}

\newcommand{\C}{\mathcal C}
\newcommand{\set}[1]{{\{ #1 \}}}

\newcommand{\topk}{\texttt{Top-k}}

\usepackage[acronym]{glossaries}
\glsdisablehyper

\usepackage{titletoc}

\usepackage{placeins}

\usepackage{tikz}
\usepackage{tikz-3dplot}

\usepackage{microtype}

\newacronym{fw}{FW}{Frank-Wolfe}
\newacronym{sparsefw}{SparseFW}{Sparse Frank-Wolfe}
\newacronym{imp}{IMP}{Iterative Magnitude Pruning}
\newacronym{per}{PERP}{Parameter-Efficient Retraining after Pruning}
\newacronym{ft}{FT}{Fine-Tuning}
\newacronym{lora}{LoRA}{Low-Rank Adaptation}
\newacronym{llm}{LLM}{Large Language Model}
\newacronym{nlp}{NLP}{Natural Language Processing}
\newacronym{bn}{BN}{Batch-Normalization}
\newacronym{ln}{LN}{Layer-Normalization}
\newacronym{lmo}{LMO}{Linear Minimization Oracle}
\newacronym{obs}{OBS}{Optimal Brain Surgeon}
\newacronym{ria}{RIA}{Relative Importance and Activations}
\newacronym{wanda}{Wanda}{Pruning by Weights and Activations}
\newacronym{sparsegpt}{SparseGPT}{SparseGPT}
\newacronym{pgd}{PGD}{Projected Gradient Descent}

\newcommand{\glsshort}[1]{\glsentryshort{#1}}

\title{Don't Be Greedy, Just Relax! Pruning LLMs via Frank-Wolfe}
\iclrfinalcopy
\author{Christophe Roux\thanks{Equal contribution} $^{\,12}$ Max Zimmer$^{*12}$ Alexandre d'Aspremont$^{3}$ Sebastian Pokutta$^{12}$\\
$^1$Department for AI in Society, Science, and Technology, Zuse Institute Berlin, Germany\\
$^2$Institute of Mathematics, Technische Universität Berlin, Germany\\
$^3$CNRS \& D.I. École Normale Supérieure, Paris, France.\\
\texttt{\{roux,zimmer,pokutta\}@zib.de}
}

\begin{document}

\maketitle
\begin{abstract}
  \emph{Pruning} is a common technique to reduce the compute and storage requirements of Neural Networks. While conventional approaches typically retrain the model to recover pruning-induced performance degradation, state-of-the-art \gls{llm} pruning methods operate layer-wise, minimizing the per-layer pruning error on a small calibration dataset to avoid full retraining, which is considered computationally prohibitive for \glspl{llm}. However, finding the optimal pruning mask is a hard combinatorial problem and solving it to optimality is intractable. Existing methods hence rely on greedy heuristics that ignore the weight interactions in the pruning objective. In this work, we instead consider the convex relaxation of these combinatorial constraints and solve the resulting problem using the \gls{fw} algorithm. Our method drastically reduces the per-layer pruning error, outperforms strong baselines on state-of-the-art GPT architectures, and remains memory-efficient. We provide theoretical justification by showing that, combined with the convergence guarantees of the \gls{fw} algorithm, we obtain an approximate solution to the original combinatorial problem upon rounding the relaxed solution to integrality.
\end{abstract}

\section{Introduction}
\emph{Pruning after training} \citep{Han2015, Gale2019, Hoefler2021, Zimmer2021, Zimmer2022} reduces the inference-time compute and memory footprint of Neural Networks with minimal impact on predictive performance. Conventional approaches obtain such \emph{sparse} models by removing parameters using simple criteria such as their magnitude and then typically require full retraining to recover pruning-induced performance degradation. The drastic increase in model size accompanying the rise of \glspl{llm} has, however, reshaped the pruning landscape.

At \gls{llm} scale, full retraining is often considered prohibitively expensive or even infeasible, resulting in a surge of interest in pruning criteria that do not require retraining. In addition, classical magnitude pruning performs no better than random pruning for \glspl{llm} \citep{Sun2023, Yin2023a}, an observation attributed to activation outliers \citep{Dettmers2022} and highly important \emph{super-weights} \citep{yuSuperWeightLarge2025} in sufficiently large \emph{Transformer} models \citep{Vaswani2017}. Consequently, state-of-the-art methods \citep{Frantar2023a, Sun2023, Zhang2024} prune \emph{layerwise}: they decompose pruning into per-layer subproblems and treat layers sequentially and independently, estimating parameter importance on a small calibration set by minimizing a per-layer \emph{local} pruning loss. Specifically, for a single layer with calibration input matrix $X\in \R^{d_{in} \times B}$ and weights $W\in \R^{d_{out} \times d_{in}}$, the objective is
\begin{equation}\label{eq:P-compression}
  \min_{M}\norm{W X - (M \odot W) X}_F^2,\quad \text{s.t. } M \in \set{0,1}^{d_{out} \times d_{in}}, \norm{M}_0 \leq k \tag{\textsc{Mask selection}}
\end{equation}
where $M \in \set{0,1}^{d_{out} \times d_{in}}$ is a binary mask that enforces the target sparsity, e.g., $\norm{M}_0 \leq k$ for unstructured pruning, and $\odot$ denotes the Hadamard product. Here, $B = N \cdot L$, where $N$ is the number of samples in the calibration batch and $L$ the sequence length. 

However, even for a single layer, selecting the optimal pruning mask is a hard quadratic binary optimization problem. Solving \labelcref{eq:P-compression} to optimality is computationally intractable at \gls{llm} scale because the combinatorial constraint—choosing $k$ out of $d_{out} \times d_{in}$ elements—results in a search space that grows exponentially with the parameter count. Prior methods such as \glsshort{sparsegpt} and \glsshort{wanda} therefore resort to greedy heuristics that ignore weight interactions to remain tractable\footnote{We discuss these methods in detail in \cref{sec:methodology}.}.

In this work, we instead consider the convex relaxation of these combinatorial constraints: we approximate \labelcref{eq:P-compression} by optimizing over the convex hull of all masks, transforming the combinatorially hard problem into a tractable convex program
\begin{equation}\label{eq:P-pruning-relaxed}
  \min_{M}\norm{W X - (M \odot W) X}_F^2,\quad \text{s.t. } M \in [0,1]^{d_{out} \times d_{in}}, \norm{M}_1 \leq k \tag{\textsc{Relaxed Mask Sel.}}
\end{equation}
where $M$ is now continuous with entries in $[0,1]$, and the cardinality constraint is replaced by an $L_1$-norm budget, see \cref{fig:mk_polytope} for a visualization. The resulting convex program can be solved efficiently using the first-order \glsentryfull{fw} algorithm \citep{LacosteJulien2013, Zeng2014, Carderera2021, Braun2022}. Notably, \gls{fw} is projection-free and moves toward extreme points of the feasible set (i.e., binary masks) via a \gls{lmo}, which is efficient to compute and naturally yields sparse updates.

\begin{figure}[t]
  \centering
  \begin{minipage}{0.45\textwidth}
    \centering
    \tdplotsetmaincoords{70}{100}
    \begin{tikzpicture}[tdplot_main_coords, scale=2.5]
      \draw[gray!60,->,dashed,very thin] (0,0,0) -- (1.25,0,0);
      \draw[gray!60,->,dashed,very thin] (0,0,0) -- (0,1.25,0);
      \draw[gray!60,->,dashed,very thin] (0,0,0) -- (0,0,1.25);

      \draw[gray!60,dashed,very thin] (1,0,1) -- (1,1,1) -- (0,1,1);
      \draw[gray!60,dashed,very thin] (1,1,1) -- (1,1,0);
      \draw[gray!60,dashed,very thin] (0,0,0) -- (1,0,0) -- (1,1,0) -- (0,1,0) -- cycle;
      \draw[gray!60,dashed,very thin] (0,0,0) -- (1,0,0) -- (1,0,1) -- (0,0,1) -- cycle;
      \draw[gray!60,dashed,very thin] (0,0,0) -- (0,1,0) -- (0,1,1) -- (0,0,1) -- cycle;

      \draw[custom_blue!60] (0,0,0) -- (1,0,0) -- (0,1,0) -- cycle;
      \draw[custom_blue!60] (0,0,0) -- (1,0,0) -- (0,0,1) -- cycle;
      \draw[custom_blue!60] (0,0,0) -- (0,1,0) -- (0,0,1) -- cycle;

      \draw[custom_blue!60] (0,0,0) -- (1,0,0);
      \draw[custom_blue!60] (0,0,0) -- (0,1,0);
      \draw[custom_blue!60] (0,0,0) -- (0,0,1);

      \filldraw[fill=custom_blue!25, draw=custom_blue!60, fill opacity=0.7] (1,0,0) -- (0,1,0) -- (0,0,1) -- cycle;

      \fill[black] (0,0,0) circle (0.5pt);
      \node[font=\scriptsize, xshift=3pt, yshift=-2pt] at (0,0,-0.05) {$(0{,}0{,}0)$};

      \fill[black] (1,0,0) circle (0.5pt);
      \node[font=\scriptsize, xshift=3pt] at (1,-0.05,-0.1) {$(1{,}0{,}0)$};

      \fill[black] (0,1,0) circle (0.5pt);
      \node[font=\scriptsize, xshift=-10pt, yshift=2pt] at (0,1,0) {$(0{,}1{,}0)$};

      \fill[black] (0,0,1) circle (0.5pt);
      \node[font=\scriptsize, xshift=2pt, yshift=6pt] at (0,0,1) {$(0{,}0{,}1)$};

      \fill[black] (1,1,0) circle (0.5pt);
      \node[font=\scriptsize, xshift=2pt, yshift=-6pt] at (1,1,0) {$(1{,}1{,}0)$};

      \fill[black] (1,0,1) circle (0.5pt);
      \node[font=\scriptsize, xshift=3pt, yshift=6pt] at (1,0,1) {$(1{,}0{,}1)$};

      \fill[black] (0,1,1) circle (0.5pt);
      \node[font=\scriptsize, xshift=-10pt, yshift=6pt] at (0,1,1) {$(0{,}1{,}1)$};

      \fill[black] (1,1,1) circle (0.5pt);
      \node[font=\scriptsize, xshift=-8pt, yshift=6pt] at (1,1,1) {$(1{,}1{,}1)$};
    \end{tikzpicture}
    \vspace{0.5em}
  \end{minipage}
  \hfill
  \begin{minipage}{0.45\textwidth}
    \centering
    \tdplotsetmaincoords{70}{100}
    \begin{tikzpicture}[tdplot_main_coords, scale=2.5]
      \draw[gray!60,->,dashed,very thin] (0,0,0) -- (1.25,0,0);
      \draw[gray!60,->,dashed,very thin] (0,0,0) -- (0,1.25,0);
      \draw[gray!60,->,dashed,very thin] (0,0,0) -- (0,0,1.25);

      \draw[gray!60,dashed,very thin] (1,0,1) -- (1,1,1) -- (0,1,1) -- cycle;
      \draw[gray!60,dashed,very thin] (1,1,1) -- (1,1,0);

      \draw[custom_blue!60] (0,0,0) -- (1,0,0) -- (1,1,0) -- (0,1,0) -- cycle;
      \draw[custom_blue!60] (0,0,0) -- (1,0,0) -- (1,0,1) -- (0,0,1) -- cycle;
      \draw[custom_blue!60] (0,0,0) -- (0,1,0) -- (0,1,1) -- (0,0,1) -- cycle;

      \filldraw[fill=custom_blue!25, draw=custom_blue!60, fill opacity=0.7] (1,1,0) -- (1,0,1) -- (0,1,1) -- cycle;
      \filldraw[fill=custom_blue!25, draw=custom_blue!60, fill opacity=0.6] (1,0,0) -- (1,0,1) -- (1,1,0) -- cycle;
      \filldraw[fill=custom_blue!25, draw=custom_blue!60, fill opacity=0.6] (0,1,0) -- (0,1,1) -- (1,1,0) -- cycle;
      \filldraw[fill=custom_blue!25, draw=custom_blue!60, fill opacity=0.6] (0,0,1) -- (1,0,1) -- (0,1,1) -- cycle;
      \draw[custom_blue!60] (1,1,0) -- (1,0,1) -- (0,1,1) -- cycle;

      \fill[black] (0,0,0) circle (0.5pt);
      \node[font=\scriptsize, xshift=3pt, yshift=-2pt] at (0,0,-0.05) {$(0{,}0{,}0)$};

      \fill[black] (1,0,0) circle (0.5pt);
      \node[font=\scriptsize, xshift=3pt] at (1,-0.05,-0.1) {$(1{,}0{,}0)$};

      \fill[black] (0,1,0) circle (0.5pt);
      \node[font=\scriptsize, xshift=-10pt, yshift=2pt] at (0,1,0) {$(0{,}1{,}0)$};

      \fill[black] (0,0,1) circle (0.5pt);
      \node[font=\scriptsize, xshift=2pt, yshift=6pt] at (0,0,1) {$(0{,}0{,}1)$};

      \fill[black] (1,1,0) circle (0.5pt);
      \node[font=\scriptsize, xshift=2pt, yshift=-6pt] at (1,1,0) {$(1{,}1{,}0)$};

      \fill[black] (1,0,1) circle (0.5pt);
      \node[font=\scriptsize, xshift=3pt, yshift=6pt] at (1,0,1) {$(1{,}0{,}1)$};

      \fill[black] (0,1,1) circle (0.5pt);
      \node[font=\scriptsize, xshift=-10pt, yshift=6pt] at (0,1,1) {$(0{,}1{,}1)$};

      \fill[black] (1,1,1) circle (0.5pt);
      \node[font=\scriptsize, xshift=-8pt, yshift=6pt] at (1,1,1) {$(1{,}1{,}1)$};
    \end{tikzpicture}
    \vspace{0.5em}
  \end{minipage}
  \caption{Visualization of $\mathcal{C}_k$ for $d_{\text{out}}=3$, $d_{\text{in}}=1$. \textbf{Left}: $k=1$, \textbf{Right}: $k=2$.}
  \label{fig:mk_polytope}
\end{figure}
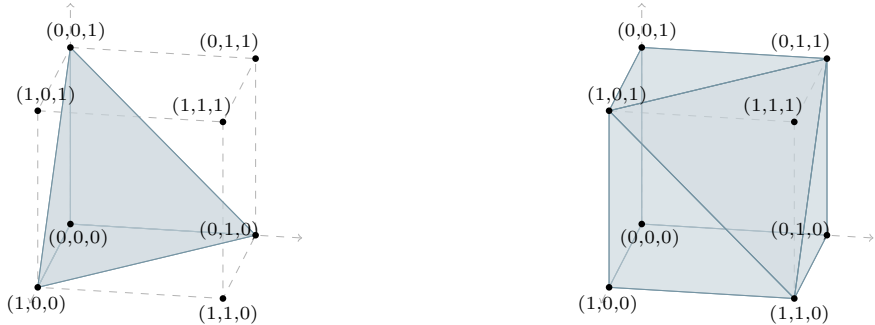

Our method, which we term \glsshort{sparsefw}, reduces the per-layer pruning error by up to 80\% compared to state-of-the-art methods such as Wanda \citep{Sun2023}, and outperforms them on benchmark GPT architectures such as Qwen 2.5, LLaMA 3, Yi 1.5, and Gemma 2, with consistent gains in final WikiText perplexity and zero-shot accuracy. \glsshort{sparsefw} is efficient, requires little memory overhead, easily adapts to unstructured and semi-structured sparsity patterns, is simple to implement, and scales to large models. Furthermore, unlike competing methods, \glsshort{sparsefw} comes with strong theoretical justification: we show that, combined with the convergence guarantees of \gls{fw}, rounding the relaxed solution to integrality yields an approximate solution to the original combinatorial problem.

\textbf{Contributions.} We summarize our contributions as follows.
\begin{enumerate}
\item \textbf{\glsshort{sparsefw}: A projection-free method for layerwise pruning.} We formulate the layerwise mask selection problem as a convex program over the convex hull of binary masks and propose to solve it with the \glsentryfull{fw} algorithm, which is projection-free and leverages an efficient \gls{lmo} that naturally yields sparse updates. \glsshort{sparsefw} is memory-efficient, simple to implement, scales to large models, and can be used to induce both unstructured and semi-structured sparsity patterns.
\item \textbf{Strong empirical performance at \gls{llm} scale.} \glsshort{sparsefw} reduces the per-layer pruning error by up to 70\% compared to state-of-the-art methods such as \glsshort{wanda}, and delivers consistent gains in final WikiText perplexity and zero-shot accuracy across modern GPT architectures (e.g., Qwen 2.5, LLaMA 3, Yi 1.5, Gemma 2).
\item \textbf{Theoretical guarantees.} We provide approximation guarantees that connect the relaxed solution returned by \gls{fw} after rounding to integrality to an approximate solution of the original combinatorial mask selection problem.
\end{enumerate}

Our work demonstrates that classical constrained optimization techniques are not only feasible for pruning \glspl{llm} but can drastically improve upon state-of-the-art performance. %

\textbf{Related work.}
\emph{Pruning after training} \citep{Hoefler2021} is among the most popular approaches to reduce the resource demands of neural networks during inference. \emph{Magnitude pruning} \citep{Janowsky1989, Han2015} is the de facto default pruning criterion for convolutional architectures, and has been shown to yield pruned models that perform competitively, despite its simplicity \citep{Gale2019, Zimmer2021}. Various other criteria exist to decide which weights to consider unimportant \citep[cf.][]{LeCun1989, Hassibi1992, Molchanov2016, Yeom2019}. With the rise of \glspl{llm}, magnitude pruning is being replaced by criteria that account for the peculiarities of \glspl{llm} \citep[in particular, large activation outliers, cf. e.g.][]{Dettmers2022, Yin2023a} and that aim to avoid requiring retraining \citep{Kwon2022, Frantar2023a, Sun2023}, which is generally considered computationally prohibitive for large models. Most importantly for our work, \glsshort{sparsegpt} \citep{Frantar2023a}, \glsshort{wanda} \citep{Sun2023}, and \glsshort{ria} \citep{Zhang2024} address the mask selection problem \labelcref{eq:P-compression} using a greedy pruning approach, where the selection of weights to prune is performed iteratively. Our approach, on the other hand, relaxes the combinatorial constraint and takes weight interactions into account.

\emph{\glsentryfull{fw}} or \emph{conditional gradient} algorithms \citep{Frank1956, Levitin1966} are widely used in Machine Learning for handling complex structural requirements efficiently \citep{LacosteJulien2013, Zeng2014, Frandi2015, Jaggi2013, Negiar2020}, with numerous theoretical works \citep{LacosteJulien2016, Hazan2016, Reddi2016} and accelerated variants \citep{Hazan2016, Yurtsever2019, Shen2019, Combettes2020, Mokhtari2018, Chen2018} appearing in the literature. For a comprehensive review, see \citet{Braun2022}. Recently, \gls{fw} has been applied in the context of neural networks \citep{Ravi2018, Xie2019, Berrada2018, Tsiligkaridis2020}, for training neural networks at scale \citep{Pokutta2020, pethick2025training}, and \citet{miao2022} as well as \citet{Zimmer2022} use \gls{fw}-variants for inducing sparsity throughout pretraining.

\section{Methodology}\label{sec:methodology}
We begin by discussing the preliminaries and demonstrating that three state-of-the-art \gls{llm} pruning methods, namely SparseGPT, Wanda, and RIA, address the mask selection problem \labelcref{eq:P-compression} using a greedy pruning approach. We then introduce the \gls{fw} algorithm and our proposed method, \glsshort{sparsefw}. Throughout this section, we use lowercase letters for scalars and vectors and uppercase letters for matrices ($W$, $X$, $M$). Matrix entries are denoted $W_{ij}$ for the element in row $i$, column $j$. Rows of matrices are denoted with lowercase subscripts: $w_i$ represents the $i$-th row of matrix $W$. We use slicing notation, e.g., $X_{j,:}$ denotes the $j$-th row of matrix $X$.

\subsection{Preliminaries and greedy methods}
Before discussing \glsshort{sparsegpt}, \glsshort{wanda}, and \glsshort{ria} in detail, we first note that the objective in \cref{eq:P-compression} decomposes into a sum of $d_{out}$ row-wise quadratic functions
\begin{equation}\label{eq:P-compression-reparam-rowwise}
  \|W X - (M \odot W) X\|_F^2= \sum_{i=1}^{d_{\text{out}}} \|\, (w_i - m_i \odot w_i)\, X \|_2^2,
\end{equation}
with $w_i \in \R^{d_{in}}$ and $m_i \in \set{0,1}^{d_{in}}$ denoting the $i$-th row of $W$ and $M$, respectively. Under unstructured sparsity, the constraint in \labelcref{eq:P-compression} couples the rows, making the problem non-separable. In contrast, semi-structured patterns such as $n\text{:}m$ (prune $M\!-\!N$ per block of $M$ weights) enforce equal per-row sparsity levels and hence fully decouple the rows. For simplicity, we will mainly discuss the row-wise formulation of \cref{eq:P-compression-reparam-rowwise} and drop the index $i$. We now analyze how \glsshort{sparsegpt}, \glsshort{wanda}, and \glsshort{ria} tackle the mask selection problem \labelcref{eq:P-compression} through greedy pruning—removing one weight at a time. These methods are optimal for their single-weight pruning objective, effectively bypassing weight interactions to simplify the problem.

\emph{\glsshort{sparsegpt}} \citep{Frantar2023a} is arguably the most popular approach and is largely based on preceding work \citep{Frantar2022a} of the authors. In practice, it prunes small blocks of weights at a time to ensure scalability to large models, instead of single weights in isolation as suggested by the theory; we briefly describe the underlying approach based on single-weight pruning. Instead of focusing solely on mask selection, \glsshort{sparsegpt} approximates the problem of finding a sparse replacement $\hat{w}$ for the weight vector $w$, thus combining the problems of mask selection and reconstruction of remaining weights by solving
\begin{equation}\label{eq:P-full-pruning}
  \min_{\hat{w}}\Vert w^\top X - \hat{w}^\top X\Vert_F^2, \quad \text{s.t. } \norm{\hat{w}}_0 \leq k.
\end{equation}
Since solving this problem exactly is intractable, SparseGPT follows a greedy procedure to approximately solve it: at each step it finds the optimal \emph{single} weight to prune and the corresponding optimal remaining weights, i.e., it solves
\begin{equation}
\min_{\hat{w}, q\in [d_{\text{in}}] \text{ s.t. } e_q^{\top}\hat{w}=0}\quad  \norm{(\hat{w} - w)^{\top} X}_2^2.
\end{equation}
The greedy-best weight index $q$ and the optimal weight reconstruction are then given by 
\[
w^* = w - \frac{w_q}{[(XX^{\top})^{-1}]_{qq}}\, (XX^{\top})^{-1} e_q, \text{ where } q \in \arg\min_{q\in [d_{\text{in}}]}\dfrac{w_q^2}{((XX^{\top})^{-1})_{qq}}.
\]

\emph{\glsshort{wanda}} \citep{Sun2023} computes a saliency score $S_{i, j} := |W_{i,j}|\,\norm{X_{j,:}}_2$ for each weight and then prunes the weights with the smallest saliencies. The authors motivate their approach by the observation that in \glspl{llm}, some weights with small magnitudes correspond to large-magnitude features \citep[cf. e.g.][]{Dettmers2022} and that their removal can lead to significant performance drops, despite their small magnitude. Wanda hence multiplies magnitude saliencies by the corresponding input activation norm to avoid pruning such small-but-important weights.

We argue that \glsshort{wanda} can be seen as a greedy approximation to \labelcref{eq:P-compression} and focus on a single row $w$ for simplicity. Again, we write the optimization problem for pruning one variable, but now without modifying the remaining weights:
\begin{equation}
\min_{\hat{w} = (1-e_q)\odot w,\,  q\in [d_{\text{in}}]}\quad \left\{\norm{(\hat{w} - w)^{\top} X}_2^2\right\}
\end{equation}
Plugging the constraints into the objective function directly yields
\begin{equation}
\min_{q\in [d_{\text{in}}]} \left\{\norm{\left ((1-e_q)\odot w) - w\right )^{\top} X}_2^2\right\}= \min_{q\in [d_{\text{in}}]} \left\{w_q^2(XX^{\top})_{qq} \right\} 
\end{equation}
Now note that $w_q^2(XX^{\top})_{qq} = w_q^2\norm{X_{q,:}}_2^2$. Minimizing the latter over $q$ is equivalent to minimizing $|w_q|\,\norm{X_{q,:}}_2$, which is exactly the saliency score of \glsshort{wanda}.

While it might seem that this procedure differs from \glsshort{wanda}, as \glsshort{wanda} computes saliency scores once for all weights and not iteratively, the approaches are identical since the saliency scores do not change after pruning a weight. \glsshort{wanda} further enforces row-wise sparsity rather than unstructured sparsity, pruning a fixed number of weights per row. This has been found beneficial for \glspl{llm} \citep{Sun2023}; the same does not hold for other transformer-like models.

\emph{\glsshort{ria}} \citep{Zhang2024} builds upon \glsshort{wanda} and uses the following saliency score: 
\begin{equation}
S^{\text{RIA}}_{ij} := | W_{ij}|\left(\frac{1}{\sum_{k=1}^{d_{\text{in}}}| W_{ik}|} + \frac{1}{\sum_{k=1}^{d_{\text{out}}}| W_{kj}|}\right)\,\|X_{j,:}\|_2.
\end{equation}
We employ full-matrix notation since \glsshort{ria} fundamentally depends on the matrix structure for its row- and column-wise renormalization. Letting $W'$ denote the rescaled weight matrix with entries
\begin{equation*}
W'_{ij}:= W_{ij}\left(\frac{1}{\sum_{k=1}^{d_{\text{in}}}| W_{ik}|} + \frac{1}{\sum_{k=1}^{d_{\text{out}}}| W_{kj}|}\right).
\end{equation*}
Applying Wanda on $W'$ to prune the weights with the smallest saliency scores yields
\begin{equation}
 |W'_{ij}|\,\|X_{j,:}\|_2 =: S^{\text{RIA}}_{ij},
\end{equation}
which is exactly the saliency score of \glsshort{ria}. The \glsshort{ria} criterion can be interpreted as using the same greedy pruning algorithm as \glsshort{wanda}, but applied to a rescaled weight matrix.

\subsection{Solving the convex relaxation with Frank-Wolfe}
We present an alternative approach to the greedy approximations discussed in the previous section, which is based on relaxing the combinatorial constraints to obtain a convex optimization problem, instead of trying to make the problem tractable by making the pruning decision on a per-weight basis. We solve the convex problem using the \gls{fw} algorithm, which we introduce in the following.

\textbf{The Frank-Wolfe Algorithm.} When minimizing some objective function $\mathcal L$ over a set of constraints $\mathcal C$, a classical approach is \gls{pgd} which iteratively performs a gradient step and then projects the result back to the constraint set to ensure feasibility of the iterates. However, depending on $\mathcal C$, this projection step may not admit an analytic solution and can be computationally expensive \citep{Jaggi2013, Combettes2021}. The \gls{fw} algorithm is an alternative which is projection-free and often yields solutions with desirable structure. Instead of moving along the gradient direction and then requiring a projection step, \gls{fw} moves towards the boundary point of the feasible region that is best aligned with the descent direction. 
Specifically, in each iteration $t$ and at iterate $M_t$, \gls{fw} calls a \glsentryfull{lmo} on the gradient $\nabla \mathcal L(M_t)$ of $\mathcal L$ at $M_t$ to solve 
\begin{equation}\label{eq:LMO}
 V_t = \argmin_{V \in \C} \langle V, \nabla \mathcal L(M_t)\rangle,
\end{equation}
which is then used to update the parameters using the convex combination
\begin{equation}\label{eq:FWUpdate}
M_{t+1} \leftarrow (1-\eta_t)M_t + \eta_t V_t,
\end{equation}
where $\eta_t \in [0,1]$ is the step size. Throughout this work, we stick to the learning rate schedule given by $\eta_t = \frac{2}{t+2}$. If now $M_0 \in \mathcal C$, then the convex update rule ensures feasibility of all iterates. In practice, solving \cref{eq:LMO} is often much cheaper than performing a projection step. If $\mathcal C$ is further given by the convex hull of a set of points, e.g., the vertices of a polytope, then the solution to \cref{eq:LMO} is attained at one of these points. In each iteration, \gls{fw} moves towards the vertices.

\textbf{Relaxing the combinatorial constraints.} The \gls{fw} algorithm can only be applied to convex constraint sets, which is not the case for \labelcref{eq:P-compression}. We make the problem tractable by relaxing the combinatorial constraints to their convex hull, i.e.,
\begin{equation}
  \label{eq:constraint_set}
  \mathcal{C}_k = \left\{M \in [0,1]^{d_{\text{out}} \times d_{\text{in}}}: \norm{M}_1 \leq k\right\}.
\end{equation}
Given that the objective function of \labelcref{eq:P-compression} is a convex quadratic, this relaxation transforms the combinatorial mask selection problem into a convex optimization problem, which can be solved efficiently using the \gls{fw} algorithm. We restate the reformulation of \labelcref{eq:P-pruning-relaxed} for completeness:
\begin{equation}
  \min_{M \in \mathcal{C}_k}\norm{W X - (M \odot W) X}_F^2.
\end{equation}
This relaxation has the advantage that, unlike the previously discussed greedy approaches, it fully accounts for interactions between weights. However, the solution to the relaxed problem \labelcref{eq:P-pruning-relaxed} is not guaranteed to be feasible for the original problem \labelcref{eq:P-compression}; in \Cref{sec:theoretical-results}, we show that rounding the relaxed solution to integrality yields an approximate solution to the original problem.

\textbf{The sparse Linear Minimization Oracle.} We next discuss how to compute the \gls{lmo} for the feasible set $\mathcal{C}_k$. Note that $\mathcal{C}_k$ is a polytope and can be described as the convex hull of its vertices, which are exactly the binary masks with at most $k$ ones. At any vertex, all coordinates lie on box bounds $0$ or $1$, and the coupling constraint $\sum_{i,j} M_{ij} \le k$ is either inactive (fewer than $k$ ones) or tight (exactly $k$ ones); see \cref{fig:mk_polytope}. Minimizing a linear function over $\mathcal{C}_k$ therefore consists of selecting up to $k$ entries with the most negative coefficients and setting them to one, leaving the rest at zero. Letting $\nabla \mathcal L(M_t) \in \mathbb{R}^{d_{\text{out}} \times d_{\text{in}}}$ denote the gradient of the objective at iterate $M_t$, the \gls{lmo} solution at step $t$ is hence given by
\begin{equation}\label{eq:lmo-mask}
[\text{LMO}\left (\nabla \mathcal L(M_t)\right )]_{ij} = \begin{cases}
  1 & \text{if } (i,j) \in \topk\left (-\nabla \mathcal L(M_t)\right ), [\nabla \mathcal L(M_t)]_{ij} < 0 \\
  0 & \text{otherwise}
  \end{cases}.
\end{equation}
where $\topk(\nabla \mathcal L(M_t))$ denotes the set of indices corresponding to the $k$ entries of $\nabla \mathcal L(M_t)$ with the smallest values. The \gls{lmo} for $\mathcal{C}_k$ can be computed efficiently and naturally produces sparse updates: at most $k$ out of $d_{\text{out}} \cdot d_{\text{in}}$ entries are nonzero. While the above corresponds to unstructured sparsity, the \gls{lmo} can be adapted to per-row sparsity and $n\text{:}m$ sparsity; see \cref{sec:structured_lmos}.

\subsection{The \glsshort{sparsefw} algorithm}
We present the full \glsshort{sparsefw} algorithm in \Cref{alg:sparsefw}. At a high level, for each layer we solve the relaxed optimization problem using the \gls{fw} algorithm, starting from any binary mask that satisfies the sparsity constraints. After running for $T$ iterations, we threshold the learned mask—whose entries lie in $[0,1]$—to obtain a binary mask that meets the original sparsity constraints. The objective function and the gradient with respect to $M_t$ are given by
\begin{align*}
  \mathcal L(M_t) &= \operatorname{Tr}(W(1-M_t)XX^{\top}(1-M_t)^{\top}W^{\top})\\
  \nabla \mathcal L(M_t) &= -2 \cdot W \odot (W XX^{\top} - (W \odot M_t) XX^{\top}).
\end{align*}
Even for small calibration datasets, the activation matrix $X$ can be very large. For example, the largest matrix in a LLaMA-2-7B transformer block (\texttt{up\_proj}) has $d_{in} = 4096$. With $N = 128$ samples and sequence length $L = 4096$, $X$ has dimensions $4096\times 524{,}288$. Because both the objective and the gradient depend only on $G := XX^{\top}$ (which can be computed in batches), we precompute $G := XX^{\top}$ and $H := WG$ once to drastically reduce resource demands. Note that $G$ has dimensions $4096\times 4096$, in contrast to the $4096\times 524{,}288$ dimensions of $X$; this independence of the sequence length $L$ and number of samples $N$ is crucial for efficiency. With $G$ and $H$ precomputed, the gradient requires only two elementwise multiplications, a matrix--matrix multiplication, and a matrix addition:
\begin{equation*}
  \nabla \mathcal L(M_t) = -2 \cdot W \odot (H - (W \odot M_t) G). 
\end{equation*}
In practice, we have to navigate a caveat that we did not detail in \Cref{alg:sparsefw} for the sake of simplicity, exact details are in the appendix. Throughout the experiments, we noticed that while \gls{fw} often substantially reduces pruning error relative to baselines like \glsshort{wanda}, it can still produce worse final perplexity, likely due to a mismatch between local and global objectives. Constraining \gls{sparsefw} by fixing a fraction of very high-saliency weights (e.g., those with highest Wanda scores) as unprunable consistently improves performance. This suggests that Wanda reliably identifies weights that should be preserved, even if a more thorough local optimization would prune them. We therefore fix these weights and apply \gls{fw} to the remaining ones, optimizing over a smaller search space. We ablate the impact of this ratio in \cref{tab:fix-ratio-ablation} in the appendix: Surprisingly, we observe the best consistent improvements when setting $\alpha=0.9$, i.e., fixing 90\% of the highest saliency weights and optimizing only over the remaining 10\%. Even small $\alpha$ values (e.g., $\alpha=0.1$) can yield significant perplexity improvements. On the other hand, setting $\alpha=0.0$ (full \gls{fw} without any fixed weights) consistently yields worse results than the baselines.

\begin{algorithm}
  \caption{\gls{sparsefw}}
  \label{alg:sparsefw}
  \begin{algorithmic}[1]
     \Require Weight matrix $W$, input $X$, no.\ of nonzero entries $k$, iterations $T$, warm-start mask $M_0$
     \vspace{0.2em}
     \hrule
     \vspace{0.2em}
  \State $G = XX^{\top}$, $H = W G$ \Comment{Precompute buffers}
  \For{$t = 0$ to $T-1$}
      \State $\nabla \mathcal L(M_{t}) = -2 \cdot W \odot (H - (W \odot M_{t}) G)$ \Comment{Compute gradient}
      \State $V_{t} = \text{LMO}\big(\nabla \mathcal L(M_{t}), \mathcal{C}_{k}\big)$ \Comment{Compute LMO}
      \State $\eta_t \leftarrow \frac{2}{t+2}$
      \State $M_{t+1} \leftarrow (1-\eta_t)M_{t} + \eta_t V_{t}$ \Comment{FW Update}
  \EndFor
  \State $[M]_{ij} \gets \begin{cases} 1 & \text{if } (i,j) \in \topk(M_T) \\ 0 & \text{otherwise} \end{cases}$ \Comment{Threshold}\label{line:threshold}\\
  \Return $M$
  \end{algorithmic}
  \end{algorithm}

\section{Experimental Results}\label{sec:experimental-results}
We present our experimental methodology; our code will be made publicly available to ensure reproducibility. Our focus is on language modeling and we utilize pretrained models from HuggingFace \citep{Wolf2020}, including \emph{LLaMA-3.1-8B} \citep{grattafioriLlama3Herd2024}, \emph{Gemma-2-9B} \citep{riviereGemma2Improving2024}, \emph{Yi-1.5-9B} \citep{youngYiOpenFoundation2025}, \emph{DeepSeek-7B-base} \citep{biDeepSeekLLMScaling2024}, and \emph{Qwen2.5-7B} \citep{yangQwen25TechnicalReport2025}. For the calibration set, we randomly sample 2048-token sequences from the \emph{C4} dataset \citep{Raffel2020a}. For validation, we select 100 sequences from the validation split. We evaluate performance using perplexity on \emph{WikiText} \citep{Merity2016} and zero-shot accuracy on the EleutherAI evaluation set \citep{Gao2023}. Following \citet{Sun2023}, we prune all linear layers with a uniform sparsity allocation across layers, while keeping the embedding and final linear head dense. \gls{sparsefw} is compared with \glsshort{wanda} and \glsshort{ria}, as these methods also aim to find a better pruning mask by solving \labelcref{eq:P-compression}; we hence do not compare directly to methods that involve a reconstruction step, such as SparseGPT \citep{Frantar2023a}. We report results for both unstructured and semi-structured sparsity \citep{Mishra2021}.

\textbf{SparseFW outperforms state-of-the-art mask selection methods.} In \Cref{tab:extended_models}, we compare \glsshort{sparsefw} (warm-started from \glsshort{wanda} or \glsshort{ria}) to the respective baselines across five state-of-the-art GPTs and multiple sparsity regimes (50\%, 60\%, and 2:4). \glsshort{sparsefw} generally performs on par with or better than the baselines in terms of perplexity; for zero-shot accuracy, \gls{sparsefw} consistently outperforms competing methods. We generally observe much more consistent and bigger improvements in the higher sparsity regimes than for 50\% sparsity.

\textbf{\gls{sparsefw} successfully optimizes the matrix-wise pruning objective.}
We observe consistent improvement in terms of the local pruning objective over both Wanda and RIA warmstarts. \Cref{fig:reconstruction_improvement} shows the per-layer reductions relative to a Wanda Warmstart, where we observe reductions of up to 80\%. In general, we found the average relative reduction over the layers to range between 20\% and 40\% across the different models, sparsity regimes and warmstarts. 

\begin{figure}[t]
  \includegraphics[width=\textwidth]{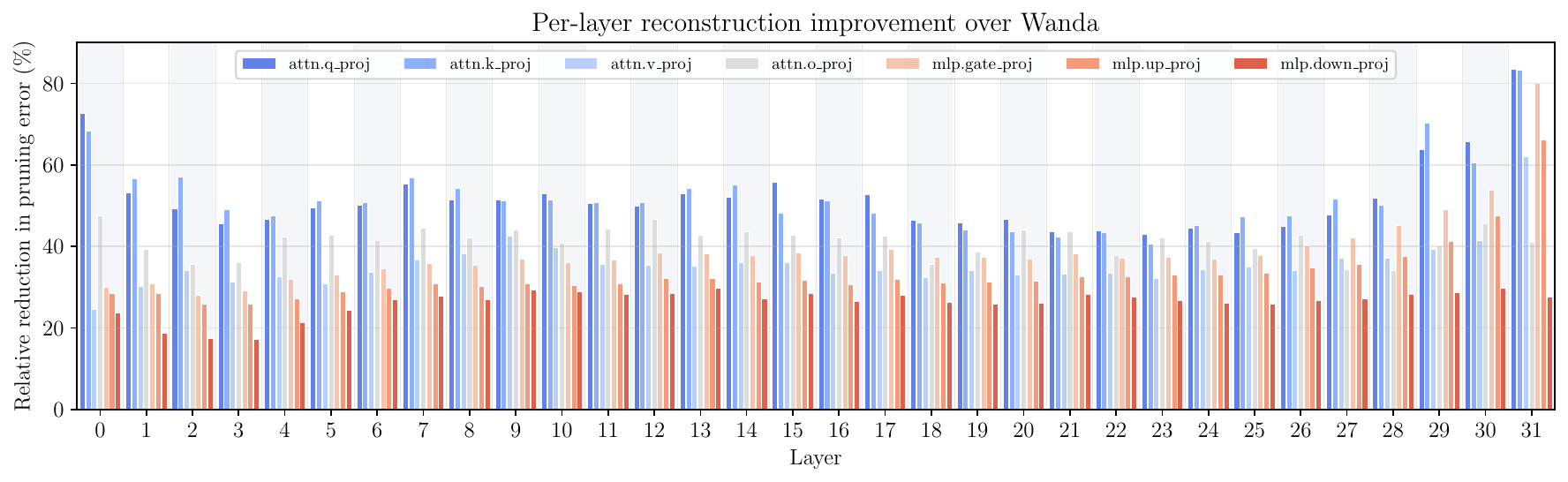}
  \caption{LLaMA-3.1-8B pruned to 60\% unstructured sparsity with \gls{sparsefw} using Wanda warmstart with 256 samples. This figure shows the relative reduction in pruning error (y-axis) for each matrix type (see legend for colors) for all layers of the model (x-axis) compared to the warmstart mask.}
        \label{fig:reconstruction_improvement}
\end{figure}

\begin{figure}[h]
  \centering
  \includegraphics[width=0.4\textwidth]{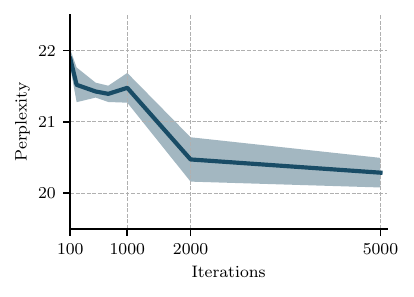}
  \includegraphics[width=0.4\textwidth]{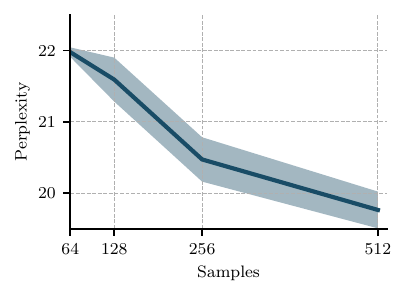}
  \caption{LLaMA-3.1-8B pruned to $2{:}4$ sparsity using \gls{sparsefw}. Left: Perplexity over the number of \gls{sparsefw} iterations per layer with 256 samples. Right: Perplexity over the number of calibration samples with 2000 \gls{sparsefw} iterations per layer. The solid curve represents the mean over multiple random seeds, the shaded regions represent the min-max range.}
  \label{fig:iter-sample-ablation}
\end{figure}

\textbf{Sample and iteration efficiency.} \Cref{fig:iter-sample-ablation} ablates the impact of the number of \gls{sparsefw} iterations (left) and the number of calibration samples (right). Fixing the amount of samples at 256, perplexity decreases up to around 2000 iterations and then flattens. We therefore use 2000 iterations throughout. In contrast, at a fixed 2000 iterations, increasing the number of calibration samples from 64 to 512 brings substantial additional perplexity gains. This trend contrasts with \glsshort{wanda}, whose performance does not seem to increase significantly with additional calibration data: increasing the sample count from 64 to 512 leads to a perplexity decrease from 25.1 to only 24.6 for \glsshort{wanda}. Overall, \gls{sparsefw} is clearly more compute-intensive than \glsshort{wanda} and \glsshort{ria}, but we argue that spending more resources once to improve the performance of pruned models is, given that deployed \glspl{llm} now serve millions of users and inference costs scale with the number of requests, worthwhile. That being said, the results of \Cref{fig:iter-sample-ablation} indicate clear benefits of increasing the number of samples while keeping the number of iterations fixed and relatively low. While more samples require slightly more compute to build the matrix $G = XX^{\top}$, the cost of a single \gls{fw} iteration is independent of the sample count.

\begin{table}
  \caption{Perplexity ($\downarrow$, lower is better) and zero-shot accuracy ($\uparrow$, higher is better) comparison. We report \gls{sparsefw} performance with Wanda and RIA warmstart for unstructured 50\% and 60\% sparsity and semi-structured $2{:}4$ sparsity after 2000 iterations using 256 samples compared to the baseline warmstarts. We indicate the \gls{sparsefw} warmstart method in parentheses. Best values are highlighted in bold. We omit standard deviations for legibility.\\ }
  \label{tab:extended_models}
  \centering
  \small
  \begin{tabular}{@{}lccccccc@{}}
\toprule
 \textbf{Perplexity} ($\downarrow$) &  & \textbf{\textsc{Gemma-2}} & \textbf{\textsc{Yi-1.5}} & \textbf{\textsc{DeepSeek-7}} & \multicolumn{2}{c}{\textbf{\textsc{Qwen2.5}}} & \textbf{\textsc{Llama-3}} \\\cmidrule{3-8}
Method &Sparsity & 9B & 9B & 7B & 7B & 14B & 8B \\
\midrule
Wanda & \multirow{4}{*}{50\%} &11.19 & 6.58 & \textbf{7.79} & 8.45 & 7.11 & 10.09 \\
RIA &  &11.19 & 6.71 & 7.90 & 8.54 & 7.01 & \textbf{9.88} \\
\gls{sparsefw} (Wanda) &  &\textbf{10.67} & 6.58 & 7.89 & 8.35 & 7.10 & 10.21 \\
\gls{sparsefw} (RIA) &  &10.77 & \textbf{6.53} & 7.93 & \textbf{8.22} & \textbf{6.98} & 9.95 \\ \midrule
Wanda & \multirow{4}{*}{60\%} &16.46 & 11.38 & \textbf{11.44} & 13.47 & 10.87 & 21.53 \\
RIA & &17.17 & 14.37 & 11.87 & 12.86 & 9.78 & 19.14 \\
\gls{sparsefw} (Wanda) &  &\textbf{14.83} & \textbf{10.56} & 11.99 & 12.44 & 10.28 & \textbf{17.97} \\ 
\gls{sparsefw} (RIA) &  &15.07 & 10.67 & 12.41 & \textbf{11.66} & \textbf{9.65} & 18.16 \\ \midrule
Wanda & \multirow{4}{*}{2:4} &17.41 & 11.58 & 11.76 & 14.40 & 11.37 & 24.82 \\
RIA &  &16.78 & 11.27 & 12.04 & 13.46 & \textbf{10.98} & 23.7 \\
\gls{sparsefw} (Wanda) &  &\textbf{15.81} & 10.61 & \textbf{11.73} & 14.16 & 11.82 & \textbf{20.45} \\
\gls{sparsefw} (RIA) &  &15.83 & \textbf{10.35} & 11.91 & \textbf{13.42} & 11.20 & 21.31 \\ \midrule
\textbf{Accuracy in \%} ($\uparrow$) &  & \textbf{\textsc{Gemma-2}} & \textbf{\textsc{Yi-1.5}} & \textbf{\textsc{DeepSeek-7}} & \multicolumn{2}{c}{\textbf{\textsc{Qwen2.5}}} & \textbf{\textsc{Llama-3}} \\\cmidrule{3-8}
Method &Sparsity & 9B & 9B & 7B & 7B & 14B & 8B \\\midrule
Wanda & \multirow{4}{*}{50\%} &68.44 & 61.04 & 56.67 & 63.72 & 67.94 & 58.78 \\
RIA &  &\textbf{68.71} & 61.22 & 55.76 & 64.03 & 67.83 & 58.94 \\
\gls{sparsefw} (Wanda) & &68.42 & 62.49 & \textbf{56.8} & 64.97 & \textbf{69.44} & \textbf{60.17} \\
\gls{sparsefw} (RIA) & &68.67 & \textbf{62.53} & 56.24 & \textbf{65.34} & 69.19 & 59.63 \\ \midrule
Wanda & \multirow{4}{*}{60\%} &63.19 & 53.7 & 50.51 & 59.44 & 63.58 & 48.08 \\
RIA & &63.19 & 53.7 & 50.51 & 59.44 & 63.58 & 48.08 \\
\gls{sparsefw} (Wanda) & &64.46 & \textbf{54.90} & 50.56 & 61.13 & 65.59 & \textbf{51.92} \\
\gls{sparsefw} (RIA) & &\textbf{65.35} & 55.41 & \textbf{50.65} & \textbf{61.52} & \textbf{65.80} & 52.15 \\ \midrule
Wanda & \multirow{4}{*}{2:4} &63.75 & 52.92 & 50.65 & 59.11 & 63.39 & 47.13 \\
RIA &  &63.83 & 52.41 & 51.08 & 58.48 & 63.85 & 47.77 \\
\gls{sparsefw} (Wanda) &  &63.81 & \textbf{53.78} & \textbf{51.12} & \textbf{60.15} & 64.12 & 48.43 \\
\gls{sparsefw} (RIA) &  &\textbf{63.90} & 52.54 & 50.69 & \textbf{60.15} & \textbf{64.35} & \textbf{48.54} \\
\bottomrule
  \end{tabular}

\end{table}

\section{Theoretical results}\label{sec:theoretical-results}
In this section, we state a data-dependent error guarantee for the mask produced by \gls{sparsefw} with respect to the original pruning objective \labelcref{eq:P-compression}. This is a key benefit of \gls{sparsefw} over greedy heuristics, which can yield suboptimal solutions even though the objective function is convex. We state our main result informally here, deferring full statements and proofs to the appendix. 

\begin{lemma}[Informal]\label{lem:sparsefw-optimality-informal}
After $T$ iterations of \gls{sparsefw}, the resulting mask $M$ satisfies
\begin{equation*}
  \mathcal L(M) -  \mathcal L(M^*) \le \lambda_{\max}\left (Q \right ) \left( \frac{k}{T} + 2\left(k + \sqrt{2 d_\text{in}d_\text{out} k}\right) \right)
\end{equation*}
where $M^*$ is an optimal mask for \labelcref{eq:P-compression}, $k$ is the maximum number of nonzeros in the mask, $Q$ represents the Hessian of the objective function and $\lambda_{\max}(Q)$ its largest eigenvalue.
\end{lemma}
Note that $Q$ is not equal to $G = XX^{\top}$, the latter being the Hessian of the objective w.r.t.\ reconstruction of the weights, not w.r.t.\ the mask. The bound captures two sources of error: (i) the \emph{optimization error} from solving the relaxed problem \labelcref{eq:P-pruning-relaxed}, and (ii) the \emph{thresholding error} from converting a relaxed solution to a binary mask (Line~\ref{line:threshold} in \cref{alg:sparsefw}).

\emph{Optimization error}. After $T$ iterations of the \gls{fw} algorithm, the resulting (continuous, not-yet-thresholded) mask $M_T$ satisfies 
\begin{equation*}
  \mathcal L(M_T)- \mathcal L(\hat M) \le k \lambda_{\max}(Q)/T,
\end{equation*}
where $\hat M$ is an optimal solution to the relaxed problem \labelcref{eq:P-pruning-relaxed}. In other words, by increasing the number of iterations $T$, \gls{fw} can guarantee an arbitrarily small optimization error.

\emph{Thresholding error}. The error due to thresholding can be controlled by the curvature of the objective (captured by $\lambda_{\max}(Q)$) and the distance between the fractional iterate and its thresholded version, which in turn can be upper bounded in terms of $k$ and the dimension of the input space $d_\text{in}d_\text{out}$. %

These insights explain the empirical behavior in \Cref{fig:normalized_fw_error_ablation}. The left panel reports the relative pruning error reduction (higher is better) versus \gls{fw} iterations for the continuous and thresholded masks. After a short initial drop, due to the large stepsize, the continuous iterate improves consistently, as predicted by the \gls{fw} convergence guarantee. In contrast, the thresholded mask first degrades as the thresholding error grows while the iterate moves through the interior of $\mathcal{C}_k$. This is reflected in the right panel, which shows the average threshold residual (the $\norm{\cdot}_1$ distance between the continuous and thresholded masks): It first rises steeply, then decreases and eventually plateaus above zero. As long as the relaxed solution is not at a vertex, the thresholding error remains nonzero, so the thresholded curve does not fully catch up to the continuous one.

\begin{figure}
\centering
\includegraphics[width=0.4\textwidth]{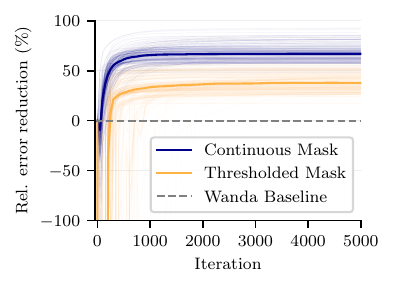}
\includegraphics[width=0.4\textwidth]{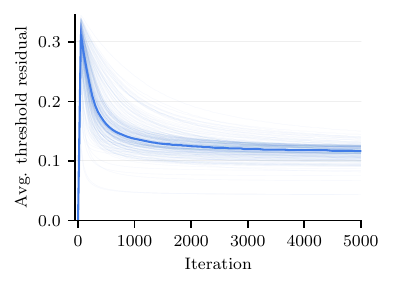}
\caption{LLaMA-3.1-8B optimized towards 60\% unstructured sparsity with \gls{sparsefw} using 256 calibration samples. Lightly colored curves show the results individual matrices; the solid curve is their median. Left: Relative pruning error reduction versus \gls{fw} iterations for continuous and thresholded masks. Right: Average threshold residual (mean $\ell_1$ distance between continuous and thresholded masks) versus iterations.}
    \label{fig:normalized_fw_error_ablation}
\end{figure}

\section{Conclusion}\label{sec:conclusion}
Solving the pruning mask selection problem for \glspl{llm} is a hard combinatorial problem. In this work, we relax the binary constraints to their convex hull and solve the resulting convex problem with the \gls{fw} algorithm; we call this approach \glsshort{sparsefw}, a simple and memory-efficient layerwise method that explicitly accounts for weight interactions and supports both unstructured and semi-structured sparsity. Across modern GPT architectures, \glsshort{sparsefw} drastically reduces the per-layer reconstruction error and improves perplexity and zero-shot accuracy over state-of-the-art \gls{llm} pruning approaches. Our work demonstrates that classical constrained optimization is a scalable and effective alternative to greedy heuristics for \gls{llm} pruning.

However, our work is not without limitations. Although vanilla \gls{fw} substantially reduces per-layer pruning error, this does not reliably yield lower perplexity. Without fixing part of the mask, it tends to prune weights crucial for overall performance. \gls{sparsefw} successfully mitigates this by preserving a fraction of high-saliency weights from the warmstart, but the local--global objective mismatch persists; inductive biases still appear necessary for improved perplexity.
\section*{Acknowledgements}
This research was partially supported by the DFG Cluster of Excellence MATH+ (EXC-2046/1, project id 390685689) funded by the Deutsche Forschungsgemeinschaft (DFG) as well as by the German Federal Ministry of Education and Research (fund number 01IS23025B).
Alexandre d'Aspremont was funded in part by the french government under management of Agence Nationale de la recherche as part of the “Investissements d’avenir” program, reference ANR-19-P3IA-0001 (PRAIRIE 3IA Institute) and a Google focused award.
\bibliography{max_references,max_zotero_references}

\begin{thebibliography}{54}
\providecommand{\natexlab}[1]{#1}
\providecommand{\url}[1]{\texttt{#1}}
\expandafter\ifx\csname urlstyle\endcsname\relax
  \providecommand{\doi}[1]{doi: #1}\else
  \providecommand{\doi}{doi: \begingroup \urlstyle{rm}\Url}\fi

\bibitem[Berrada et~al.(2018)Berrada, Zisserman, and Kumar]{Berrada2018}
Leonard Berrada, Andrew Zisserman, and M.~Pawan Kumar.
\newblock Deep frank-wolfe for neural network optimization.
\newblock \emph{International Conference on Learning Representations 2019}, November 2018.

\bibitem[Bi et~al.(2024)Bi, Chen, Chen, Chen, Dai, Deng, Ding, Dong, Du, Fu, Gao, Gao, Gao, Ge, Guan, Guo, Guo, Hao, Hao, He, Hu, Huang, Li, Li, Li, Li, Li, Liang, Lin, Liu, Liu, Liu, Liu, Liu, Liu, Lu, Lu, Luo, Ma, Nie, Pei, Piao, Qiu, Qu, Ren, Ren, Ruan, Sha, Shao, Song, Su, Sun, Sun, Tang, Wang, Wang, Wang, Wang, Wang, Wu, Wu, Xie, Xie, Xie, Xiong, Xu, Xu, Xu, Yang, You, Yu, Yu, Zhang, Zhang, Zhang, Zhang, Zhang, Zhang, Zhang, Zhang, Zhao, Zhao, Zhou, Zhou, Zhu, and Zou]{biDeepSeekLLMScaling2024}
Xiao Bi, Deli Chen, Guanting Chen, Shanhuang Chen, Damai Dai, Chengqi Deng, Honghui Ding, Kai Dong, Qiushi Du, Zhe Fu, Huazuo Gao, Kaige Gao, Wenjun Gao, Ruiqi Ge, Kang Guan, Daya Guo, Jianzhong Guo, Guangbo Hao, Zhewen Hao, Ying He, Wenjie Hu, Panpan Huang, Erhang Li, Guowei Li, Jiashi Li, Yao Li, Y.~K. Li, Wenfeng Liang, Fangyun Lin, A.~X. Liu, Bo~Liu, Wen Liu, Xiaodong Liu, Xin Liu, Yiyuan Liu, Haoyu Lu, Shanghao Lu, Fuli Luo, Shirong Ma, Xiaotao Nie, Tian Pei, Yishi Piao, Junjie Qiu, Hui Qu, Tongzheng Ren, Zehui Ren, Chong Ruan, Zhangli Sha, Zhihong Shao, Junxiao Song, Xuecheng Su, Jingxiang Sun, Yaofeng Sun, Minghui Tang, Bingxuan Wang, Peiyi Wang, Shiyu Wang, Yaohui Wang, Yongji Wang, Tong Wu, Y.~Wu, Xin Xie, Zhenda Xie, Ziwei Xie, Yiliang Xiong, Hanwei Xu, R.~X. Xu, Yanhong Xu, Dejian Yang, Yuxiang You, Shuiping Yu, Xingkai Yu, B.~Zhang, Haowei Zhang, Lecong Zhang, Liyue Zhang, Mingchuan Zhang, Minghua Zhang, Wentao Zhang, Yichao Zhang, Chenggang Zhao, Yao Zhao, Shangyan Zhou, Shunfeng Zhou, Qihao Zhu, and Yuheng Zou.
\newblock {{DeepSeek LLM}}: {{Scaling Open-Source Language Models}} with {{Longtermism}}, January 2024.
\newblock URL \url{http://arxiv.org/abs/2401.02954}.

\bibitem[Braun et~al.(2022)Braun, Carderera, Combettes, Hassani, Karbasi, Mokhtari, and Pokutta]{Braun2022}
G{\'a}bor Braun, Alejandro Carderera, Cyrille~W Combettes, Hamed Hassani, Amin Karbasi, Aryan Mokhtari, and Sebastian Pokutta.
\newblock Conditional gradient methods.
\newblock November 2022.
\newblock URL \url{https://conditional-gradients.org/}.

\bibitem[Carderera et~al.(2021)Carderera, Pokutta, Schütte, and Weiser]{Carderera2021}
Alejandro Carderera, Sebastian Pokutta, Christof Schütte, and Martin Weiser.
\newblock Cindy: Conditional gradient-based identification of non-linear dynamics -- noise-robust recovery.
\newblock January 2021.

\bibitem[Chen et~al.(2018)Chen, Harshaw, Hassani, and Karbasi]{Chen2018}
Lin Chen, Christopher Harshaw, Hamed Hassani, and Amin Karbasi.
\newblock Projection-free online optimization with stochastic gradient: From convexity to submodularity.
\newblock In \emph{International Conference on Machine Learning}, pp.\  814--823. PMLR, 2018.

\bibitem[Combettes \& Pokutta(2021)Combettes and Pokutta]{Combettes2021}
Cyrille~W. Combettes and Sebastian Pokutta.
\newblock Complexity of linear minimization and projection on some sets.
\newblock January 2021.

\bibitem[Combettes et~al.(2020)Combettes, Spiegel, and Pokutta]{Combettes2020}
Cyrille~W. Combettes, Christoph Spiegel, and Sebastian Pokutta.
\newblock Projection-free adaptive gradients for large-scale optimization.
\newblock September 2020.

\bibitem[Dettmers et~al.(2022)Dettmers, Lewis, Belkada, and Zettlemoyer]{Dettmers2022}
Tim Dettmers, Mike Lewis, Younes Belkada, and Luke Zettlemoyer.
\newblock Llm.int8(): 8-bit matrix multiplication for transformers at scale.
\newblock August 2022.

\bibitem[Frandi et~al.(2015)Frandi, Nanculef, Lodi, Sartori, and Suykens]{Frandi2015}
Emanuele Frandi, Ricardo Nanculef, Stefano Lodi, Claudio Sartori, and Johan A.~K. Suykens.
\newblock Fast and scalable lasso via stochastic frank-wolfe methods with a convergence guarantee.
\newblock October 2015.

\bibitem[Frank et~al.(1956)Frank, Wolfe, et~al.]{Frank1956}
Marguerite Frank, Philip Wolfe, et~al.
\newblock An algorithm for quadratic programming.
\newblock \emph{Naval research logistics quarterly}, 3\penalty0 (1-2):\penalty0 95--110, 1956.

\bibitem[Frantar \& Alistarh(2023)Frantar and Alistarh]{Frantar2023a}
Elias Frantar and Dan Alistarh.
\newblock Sparsegpt: Massive language models can be accurately pruned in one-shot.
\newblock In \emph{International Conference on Machine Learning}, pp.\  10323--10337. PMLR, 2023.

\bibitem[Frantar et~al.(2022)Frantar, Singh, and Alistarh]{Frantar2022a}
Elias Frantar, Sidak~Pal Singh, and Dan Alistarh.
\newblock Optimal brain compression: A framework for accurate post-training quantization and pruning.
\newblock August 2022.

\bibitem[Gale et~al.(2019)Gale, Elsen, and Hooker]{Gale2019}
Trevor Gale, Erich Elsen, and Sara Hooker.
\newblock The state of sparsity in deep neural networks.
\newblock \emph{arXiv preprint arXiv:1902.09574}, 2019.

\bibitem[Gao et~al.(2023)Gao, Tow, Abbasi, Biderman, Black, DiPofi, Foster, Golding, Hsu, Le~Noac'h, Li, McDonell, Muennighoff, Ociepa, Phang, Reynolds, Schoelkopf, Skowron, Sutawika, Tang, Thite, Wang, Wang, and Zou]{Gao2023}
Leo Gao, Jonathan Tow, Baber Abbasi, Stella Biderman, Sid Black, Anthony DiPofi, Charles Foster, Laurence Golding, Jeffrey Hsu, Alain Le~Noac'h, Haonan Li, Kyle McDonell, Niklas Muennighoff, Chris Ociepa, Jason Phang, Laria Reynolds, Hailey Schoelkopf, Aviya Skowron, Lintang Sutawika, Eric Tang, Anish Thite, Ben Wang, Kevin Wang, and Andy Zou.
\newblock A framework for few-shot language model evaluation, 12 2023.
\newblock URL \url{https://zenodo.org/records/10256836}.

\bibitem[Grattafiori et~al.(2024)Grattafiori, Dubey, Jauhri, Pandey, Kadian, {Al-Dahle}, Letman, Mathur, Schelten, Vaughan, Yang, Fan, Goyal, Hartshorn, Yang, Mitra, Sravankumar, Korenev, Hinsvark, Rao, Zhang, Rodriguez, Gregerson, Spataru, Roziere, Biron, Tang, Chern, Caucheteux, Nayak, Bi, Marra, McConnell, Keller, Touret, Wu, Wong, Ferrer, Nikolaidis, Allonsius, Song, Pintz, Livshits, Wyatt, Esiobu, Choudhary, Mahajan, {Garcia-Olano}, Perino, Hupkes, Lakomkin, AlBadawy, Lobanova, Dinan, Smith, Radenovic, Guzm{\'a}n, Zhang, Synnaeve, Lee, Anderson, Thattai, Nail, Mialon, Pang, Cucurell, Nguyen, Korevaar, Xu, Touvron, Zarov, Ibarra, Kloumann, Misra, Evtimov, Zhang, Copet, Lee, Geffert, Vranes, Park, Mahadeokar, Shah, van~der Linde, Billock, Hong, Lee, Fu, Chi, Huang, Liu, Wang, Yu, Bitton, Spisak, Park, Rocca, Johnstun, Saxe, Jia, Alwala, Prasad, Upasani, Plawiak, Li, Heafield, Stone, {El-Arini}, Iyer, Malik, Chiu, Bhalla, Lakhotia, {Rantala-Yeary}, van~der Maaten, Chen, Tan, Jenkins, Martin, Madaan, Malo, Blecher, Landzaat, de~Oliveira, Muzzi, Pasupuleti, Singh, Paluri, Kardas, Tsimpoukelli, Oldham, Rita, Pavlova, Kambadur, Lewis, Si, Singh, Hassan, Goyal, Torabi, Bashlykov, Bogoychev, Chatterji, Zhang, Duchenne, {\c C}elebi, Alrassy, Zhang, Li, Vasic, Weng, Bhargava, Dubal, Krishnan, Koura, Xu, He, Dong, Srinivasan, Ganapathy, Calderer, Cabral, Stojnic, Raileanu, Maheswari, Girdhar, Patel, Sauvestre, Polidoro, Sumbaly, Taylor, Silva, Hou, Wang, Hosseini, Chennabasappa, Singh, Bell, Kim, Edunov, Nie, Narang, Raparthy, Shen, Wan, Bhosale, Zhang, Vandenhende, Batra, Whitman, Sootla, Collot, Gururangan, Borodinsky, Herman, Fowler, Sheasha, Georgiou, Scialom, Speckbacher, Mihaylov, Xiao, Karn, Goswami, Gupta, Ramanathan, Kerkez, Gonguet, Do, Vogeti, Albiero, Petrovic, Chu, Xiong, Fu, Meers, Martinet, Wang, Wang, Tan, Xia, Xie, Jia, Wang, Goldschlag, Gaur, Babaei, Wen, Song, Zhang, Li, Mao, Coudert, Yan, Chen, Papakipos, Singh, Srivastava, Jain, Kelsey, Shajnfeld, Gangidi, Victoria, Goldstand, Menon, Sharma, Boesenberg, Baevski, Feinstein, Kallet, Sangani, Teo, Yunus, Lupu, Alvarado, Caples, Gu, Ho, Poulton, Ryan, Ramchandani, Dong, Franco, Goyal, Saraf, Chowdhury, Gabriel, Bharambe, Eisenman, Yazdan, James, Maurer, Leonhardi, Huang, Loyd, Paola, Paranjape, Liu, Wu, Ni, Hancock, Wasti, Spence, Stojkovic, Gamido, Montalvo, Parker, Burton, Mejia, Liu, Wang, Kim, Zhou, Hu, Chu, Cai, Tindal, Feichtenhofer, Gao, Civin, Beaty, Kreymer, Li, Adkins, Xu, Testuggine, David, Parikh, Liskovich, Foss, Wang, Le, Holland, Dowling, Jamil, Montgomery, Presani, Hahn, Wood, Le, Brinkman, Arcaute, Dunbar, Smothers, Sun, Kreuk, Tian, Kokkinos, Ozgenel, Caggioni, Kanayet, Seide, Florez, Schwarz, Badeer, Swee, Halpern, Herman, Sizov, Guangyi, Zhang, Lakshminarayanan, Inan, Shojanazeri, Zou, Wang, Zha, Habeeb, Rudolph, Suk, Aspegren, Goldman, Zhan, Damlaj, Molybog, Tufanov, Leontiadis, Veliche, Gat, Weissman, Geboski, Kohli, Lam, Asher, Gaya, Marcus, Tang, Chan, Zhen, Reizenstein, Teboul, Zhong, Jin, Yang, Cummings, Carvill, Shepard, McPhie, Torres, Ginsburg, Wang, Wu, U, Saxena, Khandelwal, Zand, Matosich, Veeraraghavan, Michelena, Li, Jagadeesh, Huang, Chawla, Huang, Chen, Garg, A, Silva, Bell, Zhang, Guo, Yu, Moshkovich, Wehrstedt, Khabsa, Avalani, Bhatt, Mankus, Hasson, Lennie, Reso, Groshev, Naumov, Lathi, Keneally, Liu, Seltzer, Valko, Restrepo, Patel, Vyatskov, Samvelyan, Clark, Macey, Wang, Hermoso, Metanat, Rastegari, Bansal, Santhanam, Parks, White, Bawa, Singhal, Egebo, Usunier, Mehta, Laptev, Dong, Cheng, Chernoguz, Hart, Salpekar, Kalinli, Kent, Parekh, Saab, Balaji, Rittner, Bontrager, Roux, Dollar, Zvyagina, Ratanchandani, Yuvraj, Liang, Alao, Rodriguez, Ayub, Murthy, Nayani, Mitra, Parthasarathy, Li, Hogan, Battey, Wang, Howes, Rinott, Mehta, Siby, Bondu, Datta, Chugh, Hunt, Dhillon, Sidorov, Pan, Mahajan, Verma, Yamamoto, Ramaswamy, Lindsay, Lindsay, Feng, Lin, Zha, Patil, Shankar, Zhang, Zhang, Wang, Agarwal, Sajuyigbe, Chintala, Max, Chen, Kehoe, Satterfield, Govindaprasad, Gupta, Deng, Cho, Virk, Subramanian, Choudhury, Goldman, Remez, Glaser, Best, Koehler, Robinson, Li, Zhang, Matthews, Chou, Shaked, Vontimitta, Ajayi, Montanez, Mohan, Kumar, Mangla, Ionescu, Poenaru, Mihailescu, Ivanov, Li, Wang, Jiang, Bouaziz, Constable, Tang, Wu, Wang, Wu, Gao, Kleinman, Chen, Hu, Jia, Qi, Li, Zhang, Zhang, Adi, Nam, Yu, Wang, Zhao, Hao, Qian, Li, He, Rait, DeVito, Rosnbrick, Wen, Yang, Zhao, and Ma]{grattafioriLlama3Herd2024}
Aaron Grattafiori, Abhimanyu Dubey, Abhinav Jauhri, Abhinav Pandey, Abhishek Kadian, Ahmad {Al-Dahle}, Aiesha Letman, Akhil Mathur, Alan Schelten, Alex Vaughan, Amy Yang, Angela Fan, Anirudh Goyal, Anthony Hartshorn, Aobo Yang, Archi Mitra, Archie Sravankumar, Artem Korenev, Arthur Hinsvark, Arun Rao, Aston Zhang, Aurelien Rodriguez, Austen Gregerson, Ava Spataru, Baptiste Roziere, Bethany Biron, Binh Tang, Bobbie Chern, Charlotte Caucheteux, Chaya Nayak, Chloe Bi, Chris Marra, Chris McConnell, Christian Keller, Christophe Touret, Chunyang Wu, Corinne Wong, Cristian~Canton Ferrer, Cyrus Nikolaidis, Damien Allonsius, Daniel Song, Danielle Pintz, Danny Livshits, Danny Wyatt, David Esiobu, Dhruv Choudhary, Dhruv Mahajan, Diego {Garcia-Olano}, Diego Perino, Dieuwke Hupkes, Egor Lakomkin, Ehab AlBadawy, Elina Lobanova, Emily Dinan, Eric~Michael Smith, Filip Radenovic, Francisco Guzm{\'a}n, Frank Zhang, Gabriel Synnaeve, Gabrielle Lee, Georgia~Lewis Anderson, Govind Thattai, Graeme Nail, Gregoire Mialon, Guan Pang, Guillem Cucurell, Hailey Nguyen, Hannah Korevaar, Hu~Xu, Hugo Touvron, Iliyan Zarov, Imanol~Arrieta Ibarra, Isabel Kloumann, Ishan Misra, Ivan Evtimov, Jack Zhang, Jade Copet, Jaewon Lee, Jan Geffert, Jana Vranes, Jason Park, Jay Mahadeokar, Jeet Shah, Jelmer van~der Linde, Jennifer Billock, Jenny Hong, Jenya Lee, Jeremy Fu, Jianfeng Chi, Jianyu Huang, Jiawen Liu, Jie Wang, Jiecao Yu, Joanna Bitton, Joe Spisak, Jongsoo Park, Joseph Rocca, Joshua Johnstun, Joshua Saxe, Junteng Jia, Kalyan~Vasuden Alwala, Karthik Prasad, Kartikeya Upasani, Kate Plawiak, Ke~Li, Kenneth Heafield, Kevin Stone, Khalid {El-Arini}, Krithika Iyer, Kshitiz Malik, Kuenley Chiu, Kunal Bhalla, Kushal Lakhotia, Lauren {Rantala-Yeary}, Laurens van~der Maaten, Lawrence Chen, Liang Tan, Liz Jenkins, Louis Martin, Lovish Madaan, Lubo Malo, Lukas Blecher, Lukas Landzaat, Luke de~Oliveira, Madeline Muzzi, Mahesh Pasupuleti, Mannat Singh, Manohar Paluri, Marcin Kardas, Maria Tsimpoukelli, Mathew Oldham, Mathieu Rita, Maya Pavlova, Melanie Kambadur, Mike Lewis, Min Si, Mitesh~Kumar Singh, Mona Hassan, Naman Goyal, Narjes Torabi, Nikolay Bashlykov, Nikolay Bogoychev, Niladri Chatterji, Ning Zhang, Olivier Duchenne, Onur {\c C}elebi, Patrick Alrassy, Pengchuan Zhang, Pengwei Li, Petar Vasic, Peter Weng, Prajjwal Bhargava, Pratik Dubal, Praveen Krishnan, Punit~Singh Koura, Puxin Xu, Qing He, Qingxiao Dong, Ragavan Srinivasan, Raj Ganapathy, Ramon Calderer, Ricardo~Silveira Cabral, Robert Stojnic, Roberta Raileanu, Rohan Maheswari, Rohit Girdhar, Rohit Patel, Romain Sauvestre, Ronnie Polidoro, Roshan Sumbaly, Ross Taylor, Ruan Silva, Rui Hou, Rui Wang, Saghar Hosseini, Sahana Chennabasappa, Sanjay Singh, Sean Bell, Seohyun~Sonia Kim, Sergey Edunov, Shaoliang Nie, Sharan Narang, Sharath Raparthy, Sheng Shen, Shengye Wan, Shruti Bhosale, Shun Zhang, Simon Vandenhende, Soumya Batra, Spencer Whitman, Sten Sootla, Stephane Collot, Suchin Gururangan, Sydney Borodinsky, Tamar Herman, Tara Fowler, Tarek Sheasha, Thomas Georgiou, Thomas Scialom, Tobias Speckbacher, Todor Mihaylov, Tong Xiao, Ujjwal Karn, Vedanuj Goswami, Vibhor Gupta, Vignesh Ramanathan, Viktor Kerkez, Vincent Gonguet, Virginie Do, Vish Vogeti, V{\'i}tor Albiero, Vladan Petrovic, Weiwei Chu, Wenhan Xiong, Wenyin Fu, Whitney Meers, Xavier Martinet, Xiaodong Wang, Xiaofang Wang, Xiaoqing~Ellen Tan, Xide Xia, Xinfeng Xie, Xuchao Jia, Xuewei Wang, Yaelle Goldschlag, Yashesh Gaur, Yasmine Babaei, Yi~Wen, Yiwen Song, Yuchen Zhang, Yue Li, Yuning Mao, Zacharie~Delpierre Coudert, Zheng Yan, Zhengxing Chen, Zoe Papakipos, Aaditya Singh, Aayushi Srivastava, Abha Jain, Adam Kelsey, Adam Shajnfeld, Adithya Gangidi, Adolfo Victoria, Ahuva Goldstand, Ajay Menon, Ajay Sharma, Alex Boesenberg, Alexei Baevski, Allie Feinstein, Amanda Kallet, Amit Sangani, Amos Teo, Anam Yunus, Andrei Lupu, Andres Alvarado, Andrew Caples, Andrew Gu, Andrew Ho, Andrew Poulton, Andrew Ryan, Ankit Ramchandani, Annie Dong, Annie Franco, Anuj Goyal, Aparajita Saraf, Arkabandhu Chowdhury, Ashley Gabriel, Ashwin Bharambe, Assaf Eisenman, Azadeh Yazdan, Beau James, Ben Maurer, Benjamin Leonhardi, Bernie Huang, Beth Loyd, Beto~De Paola, Bhargavi Paranjape, Bing Liu, Bo~Wu, Boyu Ni, Braden Hancock, Bram Wasti, Brandon Spence, Brani Stojkovic, Brian Gamido, Britt Montalvo, Carl Parker, Carly Burton, Catalina Mejia, Ce~Liu, Changhan Wang, Changkyu Kim, Chao Zhou, Chester Hu, Ching-Hsiang Chu, Chris Cai, Chris Tindal, Christoph Feichtenhofer, Cynthia Gao, Damon Civin, Dana Beaty, Daniel Kreymer, Daniel Li, David Adkins, David Xu, Davide Testuggine, Delia David, Devi Parikh, Diana Liskovich, Didem Foss, Dingkang Wang, Duc Le, Dustin Holland, Edward Dowling, Eissa Jamil, Elaine Montgomery, Eleonora Presani, Emily Hahn, Emily Wood, Eric-Tuan Le, Erik Brinkman, Esteban Arcaute, Evan Dunbar, Evan Smothers, Fei Sun, Felix Kreuk, Feng Tian, Filippos Kokkinos, Firat Ozgenel, Francesco Caggioni, Frank Kanayet, Frank Seide, Gabriela~Medina Florez, Gabriella Schwarz, Gada Badeer, Georgia Swee, Gil Halpern, Grant Herman, Grigory Sizov, Guangyi, Zhang, Guna Lakshminarayanan, Hakan Inan, Hamid Shojanazeri, Han Zou, Hannah Wang, Hanwen Zha, Haroun Habeeb, Harrison Rudolph, Helen Suk, Henry Aspegren, Hunter Goldman, Hongyuan Zhan, Ibrahim Damlaj, Igor Molybog, Igor Tufanov, Ilias Leontiadis, Irina-Elena Veliche, Itai Gat, Jake Weissman, James Geboski, James Kohli, Janice Lam, Japhet Asher, Jean-Baptiste Gaya, Jeff Marcus, Jeff Tang, Jennifer Chan, Jenny Zhen, Jeremy Reizenstein, Jeremy Teboul, Jessica Zhong, Jian Jin, Jingyi Yang, Joe Cummings, Jon Carvill, Jon Shepard, Jonathan McPhie, Jonathan Torres, Josh Ginsburg, Junjie Wang, Kai Wu, Kam~Hou U, Karan Saxena, Kartikay Khandelwal, Katayoun Zand, Kathy Matosich, Kaushik Veeraraghavan, Kelly Michelena, Keqian Li, Kiran Jagadeesh, Kun Huang, Kunal Chawla, Kyle Huang, Lailin Chen, Lakshya Garg, Lavender A, Leandro Silva, Lee Bell, Lei Zhang, Liangpeng Guo, Licheng Yu, Liron Moshkovich, Luca Wehrstedt, Madian Khabsa, Manav Avalani, Manish Bhatt, Martynas Mankus, Matan Hasson, Matthew Lennie, Matthias Reso, Maxim Groshev, Maxim Naumov, Maya Lathi, Meghan Keneally, Miao Liu, Michael~L. Seltzer, Michal Valko, Michelle Restrepo, Mihir Patel, Mik Vyatskov, Mikayel Samvelyan, Mike Clark, Mike Macey, Mike Wang, Miquel~Jubert Hermoso, Mo~Metanat, Mohammad Rastegari, Munish Bansal, Nandhini Santhanam, Natascha Parks, Natasha White, Navyata Bawa, Nayan Singhal, Nick Egebo, Nicolas Usunier, Nikhil Mehta, Nikolay~Pavlovich Laptev, Ning Dong, Norman Cheng, Oleg Chernoguz, Olivia Hart, Omkar Salpekar, Ozlem Kalinli, Parkin Kent, Parth Parekh, Paul Saab, Pavan Balaji, Pedro Rittner, Philip Bontrager, Pierre Roux, Piotr Dollar, Polina Zvyagina, Prashant Ratanchandani, Pritish Yuvraj, Qian Liang, Rachad Alao, Rachel Rodriguez, Rafi Ayub, Raghotham Murthy, Raghu Nayani, Rahul Mitra, Rangaprabhu Parthasarathy, Raymond Li, Rebekkah Hogan, Robin Battey, Rocky Wang, Russ Howes, Ruty Rinott, Sachin Mehta, Sachin Siby, Sai~Jayesh Bondu, Samyak Datta, Sara Chugh, Sara Hunt, Sargun Dhillon, Sasha Sidorov, Satadru Pan, Saurabh Mahajan, Saurabh Verma, Seiji Yamamoto, Sharadh Ramaswamy, Shaun Lindsay, Shaun Lindsay, Sheng Feng, Shenghao Lin, Shengxin~Cindy Zha, Shishir Patil, Shiva Shankar, Shuqiang Zhang, Shuqiang Zhang, Sinong Wang, Sneha Agarwal, Soji Sajuyigbe, Soumith Chintala, Stephanie Max, Stephen Chen, Steve Kehoe, Steve Satterfield, Sudarshan Govindaprasad, Sumit Gupta, Summer Deng, Sungmin Cho, Sunny Virk, Suraj Subramanian, Sy~Choudhury, Sydney Goldman, Tal Remez, Tamar Glaser, Tamara Best, Thilo Koehler, Thomas Robinson, Tianhe Li, Tianjun Zhang, Tim Matthews, Timothy Chou, Tzook Shaked, Varun Vontimitta, Victoria Ajayi, Victoria Montanez, Vijai Mohan, Vinay~Satish Kumar, Vishal Mangla, Vlad Ionescu, Vlad Poenaru, Vlad~Tiberiu Mihailescu, Vladimir Ivanov, Wei Li, Wenchen Wang, Wenwen Jiang, Wes Bouaziz, Will Constable, Xiaocheng Tang, Xiaojian Wu, Xiaolan Wang, Xilun Wu, Xinbo Gao, Yaniv Kleinman, Yanjun Chen, Ye~Hu, Ye~Jia, Ye~Qi, Yenda Li, Yilin Zhang, Ying Zhang, Yossi Adi, Youngjin Nam, Yu, Wang, Yu~Zhao, Yuchen Hao, Yundi Qian, Yunlu Li, Yuzi He, Zach Rait, Zachary DeVito, Zef Rosnbrick, Zhaoduo Wen, Zhenyu Yang, Zhiwei Zhao, and Zhiyu Ma.
\newblock The {{Llama}} 3 {{Herd}} of {{Models}}, November 2024.
\newblock URL \url{http://arxiv.org/abs/2407.21783}.

\bibitem[Han et~al.(2015)Han, Pool, Tran, and Dally]{Han2015}
Song Han, Jeff Pool, John Tran, and William Dally.
\newblock Learning both weights and connections for efficient neural networks.
\newblock In C.~Cortes, N.~Lawrence, D.~Lee, M.~Sugiyama, and R.~Garnett (eds.), \emph{Advances in Neural Information Processing Systems}, volume~28. Curran Associates, Inc., 2015.
\newblock URL \url{https://proceedings.neurips.cc/paper/2015/file/ae0eb3eed39d2bcef4622b2499a05fe6-Paper.pdf}.

\bibitem[Hassibi \& Stork(1993)Hassibi and Stork]{Hassibi1992}
Babak Hassibi and David Stork.
\newblock Second order derivatives for network pruning: Optimal brain surgeon.
\newblock In S.~Hanson, J.~Cowan, and C.~Giles (eds.), \emph{Advances in Neural Information Processing Systems}, volume~5. Morgan-Kaufmann, 1993.
\newblock URL \url{https://proceedings.neurips.cc/paper/1992/file/303ed4c69846ab36c2904d3ba8573050-Paper.pdf}.

\bibitem[Hazan \& Luo(2016)Hazan and Luo]{Hazan2016}
Elad Hazan and Haipeng Luo.
\newblock Variance-reduced and projection-free stochastic optimization.
\newblock In \emph{International Conference on Machine Learning}, pp.\  1263--1271. PMLR, 2016.

\bibitem[Hoefler et~al.(2021)Hoefler, Alistarh, Ben-Nun, Dryden, and Peste]{Hoefler2021}
Torsten Hoefler, Dan Alistarh, Tal Ben-Nun, Nikoli Dryden, and Alexandra Peste.
\newblock Sparsity in deep learning: Pruning and growth for efficient inference and training in neural networks.
\newblock \emph{arXiv preprint arXiv:2102.00554}, January 2021.

\bibitem[Jaggi(2013)]{Jaggi2013}
Martin Jaggi.
\newblock Revisiting frank-wolfe: Projection-free sparse convex optimization.
\newblock In \emph{Proceedings of the 30th international conference on machine learning}, pp.\  427--435, 2013.

\bibitem[Janowsky(1989)]{Janowsky1989}
Steven~A. Janowsky.
\newblock Pruning versus clipping in neural networks.
\newblock \emph{Phys. Rev. A}, 39:\penalty0 6600--6603, Jun 1989.
\newblock \doi{10.1103/PhysRevA.39.6600}.

\bibitem[Kwon et~al.(2022)Kwon, Kim, Mahoney, Hassoun, Keutzer, and Gholami]{Kwon2022}
Woosuk Kwon, Sehoon Kim, Michael~W. Mahoney, Joseph Hassoun, Kurt Keutzer, and Amir Gholami.
\newblock A fast post-training pruning framework for transformers.
\newblock March 2022.

\bibitem[Lacoste-Julien(2016)]{LacosteJulien2016}
Simon Lacoste-Julien.
\newblock Convergence rate of frank-wolfe for non-convex objectives.
\newblock July 2016.

\bibitem[Lacoste-Julien et~al.(2013)Lacoste-Julien, Jaggi, Schmidt, and Pletscher]{LacosteJulien2013}
Simon Lacoste-Julien, Martin Jaggi, Mark Schmidt, and Patrick Pletscher.
\newblock Block-coordinate frank-wolfe optimization for structural svms.
\newblock In \emph{International Conference on Machine Learning}, pp.\  53--61. PMLR, 2013.

\bibitem[LeCun et~al.(1989)LeCun, Denker, and Solla]{LeCun1989}
Yann LeCun, John~S. Denker, and Sara~A. Solla.
\newblock Optimal brain damage.
\newblock In David~S. Touretzky (ed.), \emph{Advances in Neural Information Processing Systems 2, {[NIPS} Conference, Denver, Colorado, USA, November 27-30, 1989]}, pp.\  598--605. Morgan Kaufmann, 1989.
\newblock URL \url{http://papers.nips.cc/paper/250-optimal-brain-damage}.

\bibitem[Levitin \& Polyak(1966)Levitin and Polyak]{Levitin1966}
Evgeny~S Levitin and Boris~T Polyak.
\newblock Constrained minimization methods.
\newblock \emph{USSR Computational mathematics and mathematical physics}, 6\penalty0 (5):\penalty0 1--50, 1966.

\bibitem[Merity et~al.(2016)Merity, Xiong, Bradbury, and Socher]{Merity2016}
Stephen Merity, Caiming Xiong, James Bradbury, and Richard Socher.
\newblock Pointer sentinel mixture models.
\newblock September 2016.

\bibitem[Miao et~al.(2022)Miao, Luo, Chen, Chen, Liu, and Wang]{miao2022}
Lu~Miao, Xiaolong Luo, Tianlong Chen, Wuyang Chen, Dong Liu, and Zhangyang Wang.
\newblock Learning pruning-friendly networks via frank-wolfe: One-shot, any-sparsity, and no retraining.
\newblock In \emph{International Conference on Learning Representations}, 2022.
\newblock URL \url{https://openreview.net/forum?id=O1DEtITim__}.

\bibitem[Mishra et~al.(2021)Mishra, Latorre, Pool, Stosic, Stosic, Venkatesh, Yu, and Micikevicius]{Mishra2021}
Asit Mishra, Jorge~Albericio Latorre, Jeff Pool, Darko Stosic, Dusan Stosic, Ganesh Venkatesh, Chong Yu, and Paulius Micikevicius.
\newblock Accelerating sparse deep neural networks.
\newblock April 2021.

\bibitem[Mokhtari et~al.(2018)Mokhtari, Hassani, and Karbasi]{Mokhtari2018}
Aryan Mokhtari, Hamed Hassani, and Amin Karbasi.
\newblock Conditional gradient method for stochastic submodular maximization: Closing the gap.
\newblock In \emph{International Conference on Artificial Intelligence and Statistics}, pp.\  1886--1895. PMLR, 2018.

\bibitem[Molchanov et~al.(2016)Molchanov, Tyree, Karras, Aila, and Kautz]{Molchanov2016}
Pavlo Molchanov, Stephen Tyree, Tero Karras, Timo Aila, and Jan Kautz.
\newblock Pruning convolutional neural networks for resource efficient inference.
\newblock November 2016.

\bibitem[N{\'e}giar et~al.(2020)N{\'e}giar, Dresdner, Tsai, El~Ghaoui, Locatello, Freund, and Pedregosa]{Negiar2020}
Geoffrey N{\'e}giar, Gideon Dresdner, Alicia Tsai, Laurent El~Ghaoui, Francesco Locatello, Robert Freund, and Fabian Pedregosa.
\newblock Stochastic frank-wolfe for constrained finite-sum minimization.
\newblock In \emph{International Conference on Machine Learning}, pp.\  7253--7262. PMLR, 2020.

\bibitem[Pethick et~al.(2025)Pethick, Xie, Antonakopoulos, Zhu, {Silveti-Falls}, and Cevher]{pethick2025training}
Thomas Pethick, Wanyun Xie, Kimon Antonakopoulos, Zhenyu Zhu, Antonio {Silveti-Falls}, and Volkan Cevher.
\newblock Training deep learning models with norm-constrained {{LMOs}}.
\newblock In \emph{Forty-Second International Conference on Machine Learning}, 2025.
\newblock URL \url{https://openreview.net/forum?id=2Oqm2IzTy9}.

\bibitem[Pokutta et~al.(2020)Pokutta, Spiegel, and Zimmer]{Pokutta2020}
Sebastian Pokutta, Christoph Spiegel, and Max Zimmer.
\newblock Deep neural network training with frank-wolfe.
\newblock \emph{arXiv preprint arXiv:2010.07243}, 2020.
\newblock URL \url{https://arxiv.org/abs/2010.07243}.

\bibitem[Raffel et~al.(2020)Raffel, Shazeer, Roberts, Lee, Narang, Matena, Zhou, Li, and Liu]{Raffel2020a}
Colin Raffel, Noam Shazeer, Adam Roberts, Katherine Lee, Sharan Narang, Michael Matena, Yanqi Zhou, Wei Li, and Peter~J Liu.
\newblock Exploring the limits of transfer learning with a unified text-to-text transformer.
\newblock \emph{The Journal of Machine Learning Research}, 21\penalty0 (1):\penalty0 5485--5551, 2020.

\bibitem[Ravi et~al.(2018)Ravi, Dinh, Lokhande, and Singh]{Ravi2018}
Sathya~N. Ravi, Tuan Dinh, Vishnu Lokhande, and Vikas Singh.
\newblock Constrained deep learning using conditional gradient and applications in computer vision.
\newblock March 2018.

\bibitem[Reddi et~al.(2016)Reddi, Sra, P{\'o}czos, and Smola]{Reddi2016}
Sashank~J Reddi, Suvrit Sra, Barnab{\'a}s P{\'o}czos, and Alex Smola.
\newblock Stochastic frank-wolfe methods for nonconvex optimization.
\newblock In \emph{2016 54th annual Allerton conference on communication, control, and computing (Allerton)}, pp.\  1244--1251. IEEE, 2016.

\bibitem[Riviere et~al.(2024)Riviere, Pathak, Sessa, Hardin, Bhupatiraju, Hussenot, Mesnard, Shahriari, Ram{\'e}, Ferret, Liu, Tafti, Friesen, Casbon, Ramos, Kumar, Lan, Jerome, Tsitsulin, Vieillard, Stanczyk, Girgin, Momchev, Hoffman, Thakoor, Grill, Neyshabur, Bachem, Walton, Severyn, Parrish, Ahmad, Hutchison, Abdagic, Carl, Shen, Brock, Coenen, Laforge, Paterson, Bastian, Piot, Wu, Royal, Chen, Kumar, Perry, Welty, {Choquette-Choo}, Sinopalnikov, Weinberger, Vijaykumar, Rogozi{\'n}ska, Herbison, Bandy, Wang, Noland, Moreira, Senter, Eltyshev, Visin, Rasskin, Wei, Cameron, Martins, Hashemi, {Klimczak-Pluci{\'n}ska}, Batra, Dhand, Nardini, Mein, Zhou, Svensson, Stanway, Chan, Zhou, Carrasqueira, Iljazi, Becker, Fernandez, van Amersfoort, Gordon, Lipschultz, Newlan, Ji, Mohamed, Badola, Black, Millican, McDonell, Nguyen, Sodhia, Greene, Sjoesund, Usui, Sifre, Heuermann, Lago, McNealus, Soares, Kilpatrick, Dixon, Martins, Reid, Singh, Iverson, G{\"o}rner, Velloso, Wirth, Davidow, Miller, Rahtz, Watson, Risdal, Kazemi, Moynihan, Zhang, Kahng, Park, Rahman, Khatwani, Dao, Bardoliwalla, Devanathan, Dumai, Chauhan, Wahltinez, Botarda, Barnes, Barham, Michel, Jin, Georgiev, Culliton, Kuppala, Comanescu, Merhej, Jana, Rokni, Agarwal, Mullins, Saadat, Carthy, Cogan, Perrin, Arnold, Krause, Dai, Garg, Sheth, Ronstrom, Chan, Jordan, Yu, Eccles, Hennigan, Kocisky, Doshi, Jain, Yadav, Meshram, Dharmadhikari, Barkley, Wei, Ye, Han, Kwon, Xu, Shen, Gong, Wei, Cotruta, Kirk, Rao, Giang, Peran, Warkentin, Collins, Barral, Ghahramani, Hadsell, Sculley, Banks, Dragan, Petrov, Vinyals, Dean, Hassabis, Kavukcuoglu, Farabet, Buchatskaya, Borgeaud, Fiedel, Joulin, Kenealy, Dadashi, and Andreev]{riviereGemma2Improving2024}
Morgane Riviere, Shreya Pathak, Pier~Giuseppe Sessa, Cassidy Hardin, Surya Bhupatiraju, L{\'e}onard Hussenot, Thomas Mesnard, Bobak Shahriari, Alexandre Ram{\'e}, Johan Ferret, Peter Liu, Pouya Tafti, Abe Friesen, Michelle Casbon, Sabela Ramos, Ravin Kumar, Charline~Le Lan, Sammy Jerome, Anton Tsitsulin, Nino Vieillard, Piotr Stanczyk, Sertan Girgin, Nikola Momchev, Matt Hoffman, Shantanu Thakoor, Jean-Bastien Grill, Behnam Neyshabur, Olivier Bachem, Alanna Walton, Aliaksei Severyn, Alicia Parrish, Aliya Ahmad, Allen Hutchison, Alvin Abdagic, Amanda Carl, Amy Shen, Andy Brock, Andy Coenen, Anthony Laforge, Antonia Paterson, Ben Bastian, Bilal Piot, Bo~Wu, Brandon Royal, Charlie Chen, Chintu Kumar, Chris Perry, Chris Welty, Christopher~A. {Choquette-Choo}, Danila Sinopalnikov, David Weinberger, Dimple Vijaykumar, Dominika Rogozi{\'n}ska, Dustin Herbison, Elisa Bandy, Emma Wang, Eric Noland, Erica Moreira, Evan Senter, Evgenii Eltyshev, Francesco Visin, Gabriel Rasskin, Gary Wei, Glenn Cameron, Gus Martins, Hadi Hashemi, Hanna {Klimczak-Pluci{\'n}ska}, Harleen Batra, Harsh Dhand, Ivan Nardini, Jacinda Mein, Jack Zhou, James Svensson, Jeff Stanway, Jetha Chan, Jin~Peng Zhou, Joana Carrasqueira, Joana Iljazi, Jocelyn Becker, Joe Fernandez, Joost van Amersfoort, Josh Gordon, Josh Lipschultz, Josh Newlan, Ju-yeong Ji, Kareem Mohamed, Kartikeya Badola, Kat Black, Katie Millican, Keelin McDonell, Kelvin Nguyen, Kiranbir Sodhia, Kish Greene, Lars~Lowe Sjoesund, Lauren Usui, Laurent Sifre, Lena Heuermann, Leticia Lago, Lilly McNealus, Livio~Baldini Soares, Logan Kilpatrick, Lucas Dixon, Luciano Martins, Machel Reid, Manvinder Singh, Mark Iverson, Martin G{\"o}rner, Mat Velloso, Mateo Wirth, Matt Davidow, Matt Miller, Matthew Rahtz, Matthew Watson, Meg Risdal, Mehran Kazemi, Michael Moynihan, Ming Zhang, Minsuk Kahng, Minwoo Park, Mofi Rahman, Mohit Khatwani, Natalie Dao, Nenshad Bardoliwalla, Nesh Devanathan, Neta Dumai, Nilay Chauhan, Oscar Wahltinez, Pankil Botarda, Parker Barnes, Paul Barham, Paul Michel, Pengchong Jin, Petko Georgiev, Phil Culliton, Pradeep Kuppala, Ramona Comanescu, Ramona Merhej, Reena Jana, Reza~Ardeshir Rokni, Rishabh Agarwal, Ryan Mullins, Samaneh Saadat, Sara~Mc Carthy, Sarah Cogan, Sarah Perrin, S{\'e}bastien M.~R. Arnold, Sebastian Krause, Shengyang Dai, Shruti Garg, Shruti Sheth, Sue Ronstrom, Susan Chan, Timothy Jordan, Ting Yu, Tom Eccles, Tom Hennigan, Tomas Kocisky, Tulsee Doshi, Vihan Jain, Vikas Yadav, Vilobh Meshram, Vishal Dharmadhikari, Warren Barkley, Wei Wei, Wenming Ye, Woohyun Han, Woosuk Kwon, Xiang Xu, Zhe Shen, Zhitao Gong, Zichuan Wei, Victor Cotruta, Phoebe Kirk, Anand Rao, Minh Giang, Ludovic Peran, Tris Warkentin, Eli Collins, Joelle Barral, Zoubin Ghahramani, Raia Hadsell, D.~Sculley, Jeanine Banks, Anca Dragan, Slav Petrov, Oriol Vinyals, Jeff Dean, Demis Hassabis, Koray Kavukcuoglu, Clement Farabet, Elena Buchatskaya, Sebastian Borgeaud, Noah Fiedel, Armand Joulin, Kathleen Kenealy, Robert Dadashi, and Alek Andreev.
\newblock Gemma 2: {{Improving Open Language Models}} at a {{Practical Size}}, October 2024.
\newblock URL \url{http://arxiv.org/abs/2408.00118}.

\bibitem[Shen et~al.(2019)Shen, Fang, Zhao, Huang, and Qian]{Shen2019}
Zebang Shen, Cong Fang, Peilin Zhao, Junzhou Huang, and Hui Qian.
\newblock Complexities in projection-free stochastic non-convex minimization.
\newblock In \emph{The 22nd International Conference on Artificial Intelligence and Statistics}, pp.\  2868--2876. PMLR, 2019.

\bibitem[Sun et~al.(2023)Sun, Liu, Bair, and Kolter]{Sun2023}
Mingjie Sun, Zhuang Liu, Anna Bair, and J.~Zico Kolter.
\newblock A simple and effective pruning approach for large language models.
\newblock June 2023.

\bibitem[Tsiligkaridis \& Roberts(2020)Tsiligkaridis and Roberts]{Tsiligkaridis2020}
Theodoros Tsiligkaridis and Jay Roberts.
\newblock On frank-wolfe optimization for adversarial robustness and interpretability.
\newblock December 2020.

\bibitem[Vaswani et~al.(2017)Vaswani, Shazeer, Parmar, Uszkoreit, Jones, Gomez, Kaiser, and Polosukhin]{Vaswani2017}
Ashish Vaswani, Noam Shazeer, Niki Parmar, Jakob Uszkoreit, Llion Jones, Aidan~N Gomez, {\L}ukasz Kaiser, and Illia Polosukhin.
\newblock Attention is all you need.
\newblock \emph{Advances in neural information processing systems}, 30, 2017.

\bibitem[Wolf et~al.(2020)Wolf, Debut, Sanh, Chaumond, Delangue, Moi, Cistac, Rault, Louf, Funtowicz, Davison, Shleifer, von Platen, Ma, Jernite, Plu, Xu, Le~Scao, Gugger, Drame, Lhoest, and Rush]{Wolf2020}
Thomas Wolf, Lysandre Debut, Victor Sanh, Julien Chaumond, Clement Delangue, Anthony Moi, Pierric Cistac, Tim Rault, Remi Louf, Morgan Funtowicz, Joe Davison, Sam Shleifer, Patrick von Platen, Clara Ma, Yacine Jernite, Julien Plu, Canwen Xu, Teven Le~Scao, Sylvain Gugger, Mariama Drame, Quentin Lhoest, and Alexander Rush.
\newblock Transformers: State-of-the-art natural language processing.
\newblock In \emph{Proceedings of the 2020 Conference on Empirical Methods in Natural Language Processing: System Demonstrations}, pp.\  38--45, Online, October 2020. Association for Computational Linguistics.
\newblock \doi{10.18653/v1/2020.emnlp-demos.6}.
\newblock URL \url{https://aclanthology.org/2020.emnlp-demos.6}.

\bibitem[Xie et~al.(2019)Xie, Shen, Zhang, Wang, and Qian]{Xie2019}
Jiahao Xie, Zebang Shen, Chao Zhang, Boyu Wang, and Hui Qian.
\newblock Efficient projection-free online methods with stochastic recursive gradient.
\newblock October 2019.

\bibitem[Yang et~al.(2025)Yang, Yang, Zhang, Hui, Zheng, Yu, Li, Liu, Huang, Wei, Lin, Yang, Tu, Zhang, Yang, Yang, Zhou, Lin, Dang, Lu, Bao, Yang, Yu, Li, Xue, Zhang, Zhu, Men, Lin, Li, Tang, Xia, Ren, Ren, Fan, Su, Zhang, Wan, Liu, Cui, Zhang, and Qiu]{yangQwen25TechnicalReport2025}
An~Yang, Baosong Yang, Beichen Zhang, Binyuan Hui, Bo~Zheng, Bowen Yu, Chengyuan Li, Dayiheng Liu, Fei Huang, Haoran Wei, Huan Lin, Jian Yang, Jianhong Tu, Jianwei Zhang, Jianxin Yang, Jiaxi Yang, Jingren Zhou, Junyang Lin, Kai Dang, Keming Lu, Keqin Bao, Kexin Yang, Le~Yu, Mei Li, Mingfeng Xue, Pei Zhang, Qin Zhu, Rui Men, Runji Lin, Tianhao Li, Tianyi Tang, Tingyu Xia, Xingzhang Ren, Xuancheng Ren, Yang Fan, Yang Su, Yichang Zhang, Yu~Wan, Yuqiong Liu, Zeyu Cui, Zhenru Zhang, and Zihan Qiu.
\newblock Qwen2.5 {{Technical Report}}, January 2025.
\newblock URL \url{http://arxiv.org/abs/2412.15115}.

\bibitem[Yeom et~al.(2019)Yeom, Seegerer, Lapuschkin, Binder, Wiedemann, Müller, and Samek]{Yeom2019}
Seul-Ki Yeom, Philipp Seegerer, Sebastian Lapuschkin, Alexander Binder, Simon Wiedemann, Klaus-Robert Müller, and Wojciech Samek.
\newblock Pruning by explaining: A novel criterion for deep neural network pruning.
\newblock December 2019.

\bibitem[Yin et~al.(2023)Yin, Wu, Zhang, Hsieh, Wang, Jia, Pechenizkiy, Liang, Wang, and Liu]{Yin2023a}
Lu~Yin, You Wu, Zhenyu Zhang, Cheng-Yu Hsieh, Yaqing Wang, Yiling Jia, Mykola Pechenizkiy, Yi~Liang, Zhangyang Wang, and Shiwei Liu.
\newblock Outlier weighed layerwise sparsity (owl): A missing secret sauce for pruning llms to high sparsity.
\newblock October 2023.

\bibitem[Young et~al.(2025)Young, Chen, Li, Huang, Zhang, Zhang, Wang, Li, Zhu, Chen, Chang, Yu, Liu, Liu, Yue, Yang, Yang, Xie, Huang, Hu, Ren, Niu, Nie, Li, Xu, Liu, Wang, Cai, Gu, Liu, and Dai]{youngYiOpenFoundation2025}
Alex Young, Bei Chen, Chao Li, Chengen Huang, Ge~Zhang, Guanwei Zhang, Guoyin Wang, Heng Li, Jiangcheng Zhu, Jianqun Chen, Jing Chang, Kaidong Yu, Peng Liu, Qiang Liu, Shawn Yue, Senbin Yang, Shiming Yang, Wen Xie, Wenhao Huang, Xiaohui Hu, Xiaoyi Ren, Xinyao Niu, Pengcheng Nie, Yanpeng Li, Yuchi Xu, Yudong Liu, Yue Wang, Yuxuan Cai, Zhenyu Gu, Zhiyuan Liu, and Zonghong Dai.
\newblock Yi: {{Open Foundation Models}} by 01.{{AI}}, January 2025.
\newblock URL \url{http://arxiv.org/abs/2403.04652}.

\bibitem[Yu et~al.(2025)Yu, Wang, Shan, Reed, and Wan]{yuSuperWeightLarge2025}
Mengxia Yu, De~Wang, Qi~Shan, Colorado~J. Reed, and Alvin Wan.
\newblock The {{Super Weight}} in {{Large Language Models}}, July 2025.
\newblock URL \url{http://arxiv.org/abs/2411.07191}.

\bibitem[Yurtsever et~al.(2019)Yurtsever, Sra, and Cevher]{Yurtsever2019}
Alp Yurtsever, Suvrit Sra, and Volkan Cevher.
\newblock Conditional gradient methods via stochastic path-integrated differential estimator.
\newblock In \emph{International Conference on Machine Learning}, pp.\  7282--7291. PMLR, 2019.

\bibitem[Zeng \& Figueiredo(2014)Zeng and Figueiredo]{Zeng2014}
Xiangrong Zeng and Mário A.~T. Figueiredo.
\newblock The ordered weighted $\ell_1$ norm: Atomic formulation, projections, and algorithms.
\newblock September 2014.

\bibitem[Zhang et~al.(2024)Zhang, Bai, Lin, Zhao, Hou, and Cannistraci]{Zhang2024}
Yingtao Zhang, Haoli Bai, Haokun Lin, Jialin Zhao, Lu~Hou, and Carlo~Vittorio Cannistraci.
\newblock Plug-and-play: An efficient post-training pruning method for large language models.
\newblock In \emph{The Twelfth International Conference on Learning Representations}, 2024.
\newblock URL \url{https://openreview.net/forum?id=Tr0lPx9woF}.

\bibitem[Zimmer et~al.(2023)Zimmer, Spiegel, and Pokutta]{Zimmer2021}
Max Zimmer, Christoph Spiegel, and Sebastian Pokutta.
\newblock {H}ow {I} {L}earned {T}o {S}top {W}orrying {A}nd {L}ove {R}etraining.
\newblock In \emph{International Conference on Learning Representations}, 2023.
\newblock URL \url{https://openreview.net/forum?id=_nF5imFKQI}.

\bibitem[Zimmer et~al.(2025)Zimmer, Spiegel, and Pokutta]{Zimmer2022}
Max Zimmer, Christoph Spiegel, and Sebastian Pokutta.
\newblock \emph{Compression-aware training of neural networks using Frank–Wolfe}, pp.\  137--168.
\newblock De Gruyter, Berlin, Boston, 2025.
\newblock ISBN 9783111376776.
\newblock \doi{doi:10.1515/9783111376776-010}.
\newblock URL \url{https://doi.org/10.1515/9783111376776-010}.

\end{thebibliography}
\bibliographystyle{iclr2026_conference}

\newpage
\appendix
\section{Use of Large Language Models}
Large language models were used to aid in writing (polishing text) as well as to help with the implementation of code components, including both the methods and the generation of plots. They also served as a tool for brainstorming research ideas and refining development approaches to address the challenges explored in this paper.

\section{The \gls{sparsefw} algorithm}
We state the full \glsshort{sparsefw} algorithm in \Cref{alg:full-fw-extensive}, which includes the details about how the fraction $\alpha$ of weights fixed to one is implemented. Before running \gls{fw}, we compute the number of weights to keep based on saliency $k_{\text{keep}} = \lfloor k \cdot \alpha \rfloor$ and compute the mask of the weights to keep $\overline{M}$ by setting the $k_{\text{keep}}$ weights with the highest Wanda saliency scores $S$ to one and the remaining weights to zero. Then we apply \gls{fw} to the remaining weights with the adjusted sparsity budget $k_{\text{new}} = k (1-\alpha)$. Finally, we threshold the resulting mask $M_T$ by keeping its $k_{\text{new}}$ largest entries to obtain a binary mask $M^*$, and return $M^*+\overline{M}$, which preserves the salient weights and yields exactly $k$ nonzeros.

\begin{algorithm}[h]
  \caption{The \gls{sparsefw} algorithm}
  \label{alg:full-fw-extensive}
  \begin{algorithmic}[1]
     \Require Weight matrix $W$, input data $X$, nonzero entries $k$, maximum iterations $T$, warm-start saliency matrix $S$, fraction of weights to keep from saliency $\alpha$
     \vspace{0.2em}
     \hrule
     \vspace{0.2em}
  \State $k_{\text{keep}} \leftarrow \lfloor k \cdot \alpha \rfloor$ \Comment{Number of weights retained based on saliency}
  \State $k_{\text{new}} \leftarrow \lfloor k (1-\alpha) \rfloor$ \Comment{Remaining budget}
  \State $\overline{M}_{ij} \leftarrow 1$ for $(i,j) \in \topk_{\text{keep}}(S)$, $0$ otherwise \Comment{Fixed (preserved) mask}
  \State $G = XX^{\top}$, $H = W G$ \Comment{Precompute caches}
  \For{$t = 0$ to $T-1$}
      \State $\nabla f(M_{t}) = -2 \cdot W \odot (H - (W \odot M_{t}) G)$ \Comment{Compute gradient}
      \State $V_{t} = \text{LMO}\big(\nabla f(M_{t})\odot (1-\overline{M}), \mathcal{C}_{k_{\text{new}}}\big)$ \Comment{LMO on unfixed coordinates}
      \State $\eta_t \leftarrow \frac{2}{t+2}$
      \State $M_{t+1} \leftarrow (1-\eta_t)M_{t} + \eta_t V_{t}$ \Comment{FW Update}
  \EndFor
  \State $M^*_{ij} \leftarrow 1$ if $(i,j) \in \topk_{\text{new}}(M_T)$ else $0$ \Comment{Threshold}\\
  \Return $M^*+\overline{M}$
  \end{algorithmic}
  \end{algorithm}

\FloatBarrier
\newpage

\section{Ratio of fixed weights ablation}\label{sec:fix-ratio-ablation}

\cref{tab:fix-ratio-ablation} illustrates how the ratio $\alpha$ of fixed weights impacts \gls{sparsefw} performance. Optimal results occur mostly at $\alpha=0.9$, though even a small $\alpha$ (e.g., $\alpha=0.1$) significantly enhance perplexity. Conversely, $\alpha=0.0$ (full \gls{fw} with no fixed weights) consistently underperforms compared to the baselines.

\begin{table}[h]
    \centering
    \caption{Perplexity ($\downarrow$, lower is better) comparison on WikiText. We report \gls{sparsefw} performance with after 2000 iterations using 256 samples with Wanda warmstart for unstructured 60\% sparsity and semi-structured $2{:}4$ sparsity for different ratios $\alpha$ of mask entries fixed to one (see \cref{alg:full-fw-extensive}). Here, $\alpha=1.0$ corresponds to the Wanda baseline, as no further mask entries can be optimized. Best values per row are highlighted in bold. The Wanda column provides a baseline for comparison.\\ }
    \label{tab:fix-ratio-ablation}

    \begin{tabular}{@{}lcccccccc@{}}
      \toprule
      & & \multicolumn{7}{c}{\textbf{$\alpha$-ratio of fixed weights}} \\
      \cmidrule{3-9}
          Model &Sparsity      & 0.0 & 0.1 & 0.25 & 0.5 & 0.75 & 0.9 & \multicolumn{1}{|c}{1.0 (Wanda)}\\ \midrule
        Gemma-2-9B & 2:4& 17.70 & 16.69 & 16.78 & 16.48 & 15.99 & \textbf{15.81} & \multicolumn{1}{|c}{17.41} \\
        Yi-1.5-9B & 2:4&12.26 & 11.50 & 11.49 & 11.25 & 10.83 & \textbf{10.61} & \multicolumn{1}{|c}{11.58} \\
        DeepSeek-7B & 2:4& 13.25 & 12.77 & 13.13 & 12.99 & 12.32 & \textbf{11.73} & \multicolumn{1}{|c}{11.76} \\
        Qwen2.5-7B & 2:4&16.16 & 14.96 & 15.06 & 15.21 & 14.59 & \textbf{14.16} & \multicolumn{1}{|c}{14.40} \\
        Qwen2.5-14B & 2:4& 13.70 & 12.62 & 13.34 & 12.99 & 12.79 & 11.82 & \multicolumn{1}{|c}{\textbf{11.37}} \\
        Llama-3.1-8B & 2:4& 21.95 & 20.47 & \textbf{20.45} & 21.77 & 21.73 & 21.49 & \multicolumn{1}{|c}{24.82} \\\midrule
        Gemma-2-9B &60\%& 18.25 & 16.41 & 15.78 & 15.46 & 14.92 & \textbf{14.83} & \multicolumn{1}{|c}{16.46} \\
        Yi-1.5-9B &60\%& 11.19 & \textbf{10.56} & 10.66 & 10.81 & 11.06 & 11.31 & \multicolumn{1}{|c}{11.38} \\
        DeepSeek-7B &60\%& 12.49 & 11.99 & 12.06 & 12.19 & 12.20 & 12.21 & \multicolumn{1}{|c}{\textbf{11.44}} \\
        Qwen-7B &60\%& 14.28 & 13.13 & 13.12 & 12.73 & \textbf{12.44} & 12.54 & \multicolumn{1}{|c}{13.47} \\
        Qwen-14B &60\%& 11.59 & 10.52 & 10.48 & 10.61 & 10.29 & \textbf{10.28} & \multicolumn{1}{|c}{10.87} \\
        Llama-3.1-8B &60\%& 22.47 & 18.96 & \textbf{17.97} & 18.04 & 18.27 & 19.07 & \multicolumn{1}{|c}{21.53} \\
      \bottomrule
      \end{tabular}
    \end{table}

\FloatBarrier
  \section{LMOs for Semi-Structured Sparsity}\label{sec:structured_lmos}
  Recall the definition of the constraint set $\mathcal{C}_k$ from \Cref{eq:constraint_set} for the unstructured sparsity case:
  \begin{equation*}
    \mathcal{C}_k = \left\{M \in [0,1]^{d_{\text{out}} \times d_{\text{in}}}: \norm{M}_1 \leq k\right\}.
  \end{equation*}
  For the $n\text{:}m$ sparsity case, which corresponds to keeping at most $m$ nonzeros in every group of $n$ consecutive entries of each row, and assuming $d_{\text{in}}$ is divisible by $n$, we can write the constraint set as
  \begin{equation*}
    \mathcal{C}_{n:m} = \left\{ M \in [0,1]^{d_{\text{out}} \times d_{\text{in}}} \bigm| \sum_{j=qn+1}^{(q+1)n} M_{i,j} \leq m,\ \forall i,\ \forall q \in \left\{0,\dots,d_{\text{in}}/n - 1\right\} \right\}.
  \end{equation*}
  Notice that this constraint set is simply the cartesian product of the constraint set for each block of $n$ consecutive entries of each row, which can be written as
  \begin{equation*}
    \mathcal{C}' = \left\{ M' \in [0,1]^{n}: \norm{M'}_1 \leq m \right\},
  \end{equation*}
  which is a special case of the polytope $\mathcal{C}_k$ when $d_{\text{out}}=n$, $d_{\text{in}}=1$ and $k=m$.
  Since we know the \gls{lmo} for $\mathcal{C}_k$ and the LMO problem is fully separable between the $\mathcal{C}'$ sets, we can simply apply the \gls{lmo} for $\mathcal{C}_k$ to each set $\mathcal{C}'$ individually to obtain the \gls{lmo} for $\mathcal{C}_{n:m}$. 

\section{Theoretical guarantee for \gls{sparsefw}}
For simplicity, we work in the row-wise formulation; the proof for the full-matrix case follows by the same arguments.
Let us introduce the relevant notation and definitions for the row-wise formulation.
We first fix $w\in\mathbb{R}^d_\text{in}$ (a row of $W$) and $X\in\mathbb{R}^{d_\text{in}\times B}$.
For $m\in\mathcal{C}_k$ as defined in \cref{eq:constraint_set}, the objective function is
\[
f(m):=\norm{w^\top X-(w\odot m)^\top X}_2^2
=(1-m)^\top Q(1-m),
\]
where $Q:=\mathrm{Diag}(w)\,(XX^\top)\,\mathrm{Diag}(w)\succeq0$. Let $\lambda_{\max}(Q)$ denote the top eigenvalue of $Q$.
We denote the combinatorial constraint of the original problem \cref{eq:P-compression} as
\[
\mathcal{C}_{\rm int}:=\Big\{m\in\{0,1\}^{d_{\text{in}}} \bigm|\sum_j m_j=k\Big\}.\]
Now we denote by $m^*$ the solution to the relaxed problem \labelcref{eq:P-pruning-relaxed} and by $m^{\rm int}$ the solution to the integral problem \labelcref{eq:P-compression}.

\begin{lemma}[Formal statement of \cref{lem:sparsefw-optimality-informal}]
\label{lem:sparsefw-optimality-formal}
Let $m^\varepsilon\in\mathcal{C}_k$ satisfy $\sum_j m^\varepsilon_j=k$ and 
$f(m^\varepsilon)\le f(m^*)+\varepsilon$.
Let $\widehat m:=\mathbf 1\{\,j\in \topk(m^\varepsilon)\,\}$ be its top-$k$ rounding.
Then, with $r:=d_\text{in}-k$,
\begin{equation} \label{eq:general-bound}
f(\widehat m)-f(m^{\rm int})
\le
\varepsilon\ +\ 2\,\lambda_{\max}(Q)\,\Big(\min\{k,r\}+\sqrt{2r\,\min\{k,r\}}\Big).
\end{equation}
Note that for sparsity 50\% or more, we have $2k\le d_\text{in}$ and hence $\min\{k,r\}=k$, it follows that 
\begin{equation} \label{eq:loose-bound}
f(\widehat m)-f(m^{\rm int})
\le
\varepsilon + 2\lambda_{\max}(Q) ( k+ \sqrt{2 d_\text{in}\cdot k}).
\end{equation}
\end{lemma}

\begin{proof}[Proof of \cref{lem:sparsefw-optimality-formal}]
Our goal is to bound $f(\widehat m)-f(m^{\rm int})$. To that end, first note that 
\begin{equation}
  f(m^\varepsilon)\le f(m^*)+\varepsilon\le f(m^{\rm int})+\varepsilon,
\end{equation}
where the first inequality follows by assumption on $m^\varepsilon$ and the second inequality follows since by the optimality of $m^*$ we have $f(m^*)\le f(m^{\rm int})$ (restricting to the $\mathcal{C}_{\rm int}$ can only make the objective worse). Therefore it suffices to bound $f(\widehat m)-f(m^\varepsilon)$.

Set $v:=\widehat m-m^\varepsilon$ and $z^\varepsilon:=\mathbf 1-m^\varepsilon$.
By construction, $\sum_j \widehat m_j=\sum_j m^\varepsilon_j=k$, hence $\mathbf 1^\top v=0$.
Let
\[
\tau := \sum_{j\notin \topk(m^\varepsilon)} m^\varepsilon_j
 = k-\sum_{j\in \topk(m^\varepsilon)} m^\varepsilon_j.
\]
Then we have that
\[
\begin{aligned}
f(\widehat m)-f(m^\varepsilon)
&=(z^\varepsilon - v)^\top Q(z^\varepsilon - v)-(z^\varepsilon)^\top Q z^\varepsilon\\
&= v^\top Q v\;-\;2\,z^{\varepsilon\top} Q v\\
&\le \lambda_{\max}(Q)\,\norm{v}_2^2 + 2 \lambda_{\max}(Q)\, \norm{z^\varepsilon}_2\,\norm{v}_2\\
&\le \lambda_{\max}(Q)\big(2\tau\big) + 2\lambda_{\max}(Q)\sqrt{r}\sqrt{2\tau},
\end{aligned}
\]
where the equalities follow by defintion of $f$ and the first inequality follows by Cauchy-Schwarz. For the second inequality, consider that we have that $\norm{v}_1=2\tau$ and $|v_j|\le 1$, hence $\norm{v}_2^2\le \norm{v}_1=2\tau$.
Further, we have that $\norm{z^\varepsilon}_2^2\le \norm{z^\varepsilon}_1=\sum_j(1-m^\varepsilon_j)=d_\text{in}-k=r$.

Lastly, we note that $\tau \le \min\{k,r\}$. This holds since  we have that $\tau \le \sum_j m_j = k$ and 
\[\tau =\sum_{j\notin \topk(m^\varepsilon)} m^\varepsilon_j\le \sum_{j\notin \topk(m^\varepsilon)}1\le d_\text{in}-k\] 
where the inequality follows since each $m_j \le 1$, and there are at most $d_\text{in}-k$ terms in that sum. 
This concludes the proof for the \cref{eq:general-bound} and the proof of the \cref{eq:loose-bound} follows by simple computations.

\end{proof}

\end{document}